\documentclass{article}

\usepackage[utf8]{inputenc}
\usepackage[blocks]{authblk}
\makeatletter
\let\maketitle\AB@maketitle
\makeatother
\usepackage{amsmath}
\usepackage{amssymb}
\usepackage{graphicx}
\usepackage[dvipsnames]{xcolor}
\usepackage{csquotes}
\usepackage[style=alphabetic,natbib=true,maxalphanames=3,minalphanames=3,maxbibnames=99]{biblatex}
\usepackage{hyperref}
\usepackage{cleveref}
\usepackage{mathtools}
\usepackage{color-edits}
\addauthor[Mahdi]{m}{blue}
\usepackage{amsthm}
\usepackage{thmtools}
\usepackage{thm-restate}
\usepackage[margin=1in]{geometry}
\usepackage{booktabs}
\usepackage{tabularx}
\usepackage{array}
\usepackage{multirow}
\usepackage{multirow}
\usepackage{algorithm}
\usepackage{algorithmic}
\usepackage{subcaption}

\usepackage{tikz}
\usetikzlibrary{arrows.meta, positioning, quotes}
\usepackage{xcolor}
\usepackage{dsfont}
\newcommand{\IND}{\mathds{1}}
\newcommand{\ind}[1]{{\IND \left\{ #1 \right\}}}
\DeclarePairedDelimiter\norm{\lVert}{\rVert}%

\newcolumntype{L}[1]{>{\raggedright\arraybackslash}p{#1}}
\newcolumntype{C}[1]{>{\centering\arraybackslash}p{#1}}
\newcolumntype{Z}{>{\raggedright\arraybackslash}X}

\newtheorem{definition}{Definition}
\newtheorem{theorem}{Theorem}
\newtheorem{lemma}{Lemma}

\newtheorem{proposition}{Proposition}

\crefname{assumption}{Assumption}{Assumptions}

\newcommand{\E}{\mathbb{E}}
\newcommand{\R}{\mathbb{R}}
\newcommand{\cA}{\mathcal{A}}
\newcommand{\cD}{\mathcal{D}}
\newcommand{\cE}{\mathcal{E}}

\newcommand{\cL}{\mathcal{L}}
\newcommand{\Pa}{\operatorname{Pa}}
\newcommand{\MSE}{\operatorname{MSE}}

\AddToHook{cmd/appendix/before}{\crefalias{section}{appendix} \crefalias{subsection}{appendix}}

\title{Optimal Rates for Agentic Networked Information Aggregation}
\newcommand{\authorbox}[1]{\makebox[0.38\textwidth][c]{#1}}
\author{\authorbox{MohammadHossein~Bateni}}
\affil{Google Research\\\texttt{bateni@google.com}}
\author{\authorbox{Zahra~Hadizadeh}}
\affil{University of California, Irvine\\\texttt{zhadizad@uci.edu}}
\author{\authorbox{MohammadTaghi~Hajiaghayi}}
\affil{University of Maryland\\\texttt{hajiagha@umd.edu}}
\author{\authorbox{Mahdi~JafariRaviz}}
\affil{University of Maryland\\\texttt{mahdij@umd.edu}}
\author{\authorbox{Shayan~Taherijam}}
\affil{University of California, Irvine\\\texttt{staherij@uci.edu}}
\date{}

\begin{document}
\maketitle

\begin{abstract}
Building on the pioneering paper of \citet{kearns2026networked} (SODA'26), we study
information aggregation in a networked learning model. The model captures a central pattern in agentic AI: each agent sees only part of the data and passes on only its own conclusion.
Their model considers a linear regression problem with the mean squared error (MSE) loss. Agents sit in a DAG and each sees only a subset of the features and its parents'
predictions, fits a linear predictor, and passes only its prediction forward. The
benchmark is the full-feature learner that sees all raw features. A path of depth
$D$ is $M$-covered if every block of $M$ consecutive agents collectively sees all
raw features. \citet{kearns2026networked} proved that the excess mean squared
error of the last agent on such a path is $O(M/\sqrt D)$, and gave a cyclic
instance with excess error $\Omega(M/D)$ for $D<M^2$.

We close this gap: the correct rate is constant up to depth $M^2$, and $\Theta(M^2/D)$ beyond it. We first give a sharper analysis of the cyclic instance and improve its lower bound to $\Omega(\sqrt{M/D})$ for $D<M^2$. We then construct, for every depth $D\ge M^2$, an $M$-covered
path of depth $D$ with excess error
$\Omega(M^2/D)$. The same instance gives the constant lower bound for all $D < M^2$. We also show that for any fixed distribution the excess error contracts geometrically along the path, ruling out any single instance that witnesses any polynomial lower bound at every depth.

Finally, we prove the same optimal rate for logistic classification in the logit-passing model of \citet{bateni2026networked}, which considers the binary cross-entropy (BCE) loss. The same improved upper bound of $O(M^2/D)$ holds, and we transfer all the regression lower bounds by showing that on those examples the logistic path follows the least-squares path up to rescaling.
\end{abstract}

\section{Introduction}

AI systems increasingly run as networks of agents rather than as one model. In a typical pipeline, no agent sees all the data. Each agent works on its own slice, limited by context windows, privacy, or cost, and passes forward only a short output such as a prediction \citep{guo2024multiagent}. For example, the Chain-of-Agents framework of \citet{zhang2024chain} splits a long input across a chain of language-model agents, and each agent sends a single message to the next. The pattern is older than these systems: in social learning, going back to \citet{degroot1974reaching}, parties publish predictions instead of sharing what they saw, and in vertical federated learning, parties hold different feature columns of the same data \citep{yang2019federated}. In all these settings, the same question comes up: can an agent late in the network predict as well as one centralized learner that sees all the data? Since each agent forwards only its prediction, information about the data builds up slowly, one hop at a time. Depth is what makes good prediction possible, but it is also the cost, since each hop is another trained model or model call. The real question is then a quantitative one: how fast does the error fall as depth grows?

\citet{kearns2026networked} made this question precise in a distributed learning model where agents $A_1, \dots, A_N$ sit in a directed acyclic graph and act in a topological order. There is a distribution $\mathcal{D}$ over the features $x_1, \dots, x_d$ and the label $Y$. Each agent $A_i$ sees only a subset of the features $S_i \subseteq [d]$, and the predictions of preceding agents that have edges to it. Agent $A_i$ fits a linear predictor $f_i$ of the label that minimizes the mean squared error using all its inputs, and passes its prediction to its successors. Let $\MSE(f)$ be the mean squared error of the predictor $f$. To measure the performance of this model, they consider the excess error $\MSE(f_N) - \MSE(f^\star)$ of the final agent's prediction $f_N$ compared to the best possible linear predictor $f^\star$, where a single agent would have access to all the features $x_1, \dots, x_d$.

To analyze this model, \citet{kearns2026networked} considered two metrics: the depth of the graph and the coverage of the features along the path. If $A_1, \dots, A_D$ is a directed path of agents in the graph, we say that the path is $M$-covered if every block of $M$ consecutive agents sees all features. \citet{kearns2026networked} showed that the excess error of $f_D$ is bounded by $O(M/\sqrt{D})$ when there is an $M$-covered path of depth $D$ to the final agent. They also provided an example (which we call the \emph{cyclic example}) and showed that it has an excess error of $\Omega(M/D)$. This lower bound shows that the depth of the graph is a necessary condition for good performance, but leaves a large gap in the rate.

\citet{bateni2026networked} also studied the same model, but instead of considering the regression problem with the least-squares loss, they considered the binary classification problem with the binary cross-entropy (BCE) loss. Each agent fits a linear logit predictor to its inputs, and the logit $z$ also goes through the sigmoid function, so the final prediction is the probability $\sigma(z)$. Agents send only the logit values (and not the probabilities) to their successors. \citet{bateni2026networked} showed that the same upper and lower bounds hold for the classification problem.

\subsection{Our Contributions}

We will show that most bounds given by \citet{kearns2026networked} and \citet{bateni2026networked} are far from tight, and close the gaps between the upper and lower bounds. \Cref{tab:bounds-summary} summarizes our bounds next to the previous ones. We will state our results in the regression setting first, and later show analogous results for classification.

\begin{table}[t]
\centering
\footnotesize
\setlength{\tabcolsep}{4pt}
\setlength{\arrayrulewidth}{0.4pt}

\caption{
Bounds on the excess error of the last agent on an $M$-covered
path of depth $D$.
}
\label{tab:bounds-summary}

\begingroup
\renewcommand{\arraystretch}{1.25}
\setlength{\extrarowheight}{2pt}

\begin{tabular}{C{0.14\textwidth}|L{0.17\textwidth}@{\quad}C{0.10\textwidth}|C{0.26\textwidth}|C{0.24\textwidth}}
\textbf{Setting}
&
\textbf{Result}
&
\textbf{Regime}
&
\textbf{Previous bound}
&
\textbf{This work}
\\
\hline

\multirow{5}{*}{Regression}
&
Upper bound
&
All $D$
&
$O(M/\sqrt D)$ \cite{kearns2026networked}
&
$O(M^2/D)$
\\
\cline{2-5}

&
\multirow{2}{*}{Lower bound}
&
$D<M^2$
&
$\Omega(M/D)$ \cite{kearns2026networked}
&
$\Omega(1)$
\\

&
&
$D\ge M^2$
&
--
&
$\Omega(M^2/D)$
\\
\cline{2-5}

&
Cyclic example
&
$D<M^2$
&
$\Omega(M/D)$ \cite{kearns2026networked}
&
$\Omega(\sqrt{M/D})$
\\
\cline{2-5}

&
Fixed distribution
&
All $D$
&
--
&
$O\left(q^{\lfloor(D-1)/M\rfloor}\right)$
\newline
for some $q<1$
\\
\hline

\multirow{5}{*}{Classification}
&
Upper bound
&
All $D$
&
$O(M/\sqrt D)$ \cite{bateni2026networked}
&
$O(M^2/D)$
\\
\cline{2-5}

&
\multirow{2}{*}{Lower bound}
&
$D<M^2$
&
$\Omega(M/D)$ \cite{bateni2026networked}
&
$\Omega(1)$
\\

&
&
$D\ge M^2$
&
--
&
$\Omega(M^2/D)$
\\
\cline{2-5}

&
Cyclic example
&
$D<M^2$
&
$\Omega(M/D)$ \cite{bateni2026networked}
&
$\Omega(\sqrt{M/D})$
\\
\cline{2-5}

&
Fixed distribution
&
All $D$
&
--
&
$O\left(q^{\lfloor(D-1)/M\rfloor}\right)$
\newline
for some $q<1$
\\

\end{tabular}

\endgroup
\end{table}

\subsubsection{Regression}

We show that the earlier upper bound of $O(M/\sqrt{D})$ by \citet{kearns2026networked} is not tight in the following theorem. We use the exact same assumptions as the earlier upper bound.

\begin{restatable}[Improved regression upper bound]{theorem}{regressionUpperRestatement}
\label{thm:regression-improved-upper}
Consider an $M$-covered path of depth $D$. Assume the global predictor
$f^\star(x)=\sum_{\ell=1}^d w^\star_\ell x_\ell$ satisfies
$\sum_\ell |w^\star_\ell|\le A^\star$, and each feature satisfies
$\E[x_\ell^2]\le M_X^2$. Then
\begin{equation*}
  \MSE(f_D)-\MSE(f^\star)
  \le
  C\frac{M^2}{D},
\end{equation*}
where $C$ depends only on $A^\star$ and $M_X$.
\end{restatable}

\begin{proof}[Proof Sketch]
To explain where the gain comes from, we briefly recall the argument of \citet{kearns2026networked}. The error is non-increasing along the path, so the total error drop over the path is at most the first agent's excess error, which is bounded by $\MSE(0) - \MSE(f^\star)$ which itself is bounded by a constant depending only on $A^\star$ and $M_X$. Splitting the path into $\lfloor D/M\rfloor$ blocks of $M$ consecutive agents, the pigeonhole principle finds a block whose drop is $O(M/D)$. A small drop over a block forces a small excess error at its end. Quantitatively, a drop of $\varepsilon$ over a block gives excess error $O(\sqrt{M\varepsilon})$, and $\varepsilon=O(M/D)$ gives the rate $O(M/\sqrt D)$.

We instead feed this argument its own output. Split the path in half. By induction on the depth, the agent at the midpoint already has excess error $O(M^2/D)$. The drop over the second half is at most this excess error, not a constant, so the pigeonhole now finds a block in the second half with drop $O(M^3/D^2)$, and the block bound $O(\sqrt{M\varepsilon})$ gives excess error $O(M^2/D)$ at the final agent, closing the induction.
\end{proof}

This improves on $O(M/\sqrt D)$, but only once $D>M^2$. Before that point,
$M^2/D$ (or the previous $M/\sqrt{D}$) is at least constant. We next turn to the lower-bound path of \citet{kearns2026networked}, which we call the \emph{cyclic example}. For a fixed $k$, it considers independent standard Gaussian variables $z_1,\ldots,z_k$, and builds $k$ features as $X_1 = z_1$ and $X_i = z_i - z_{i-1}$ for $i \ge 2$. The label is $Y=z_k$, and the global predictor is exact since $\sum_i X_i = z_k$. Each agent sees one feature, in the repeating cyclic order $X_1,\ldots,X_k,X_1,\ldots,X_k,\ldots$, and a \emph{pass} is one full cycle. The path has
$M=k$ and $D=pk$ after $p\le k-1$ passes, so it only tests the range $D<M^2$. In this range, the previous bound on the excess error is $\Omega(M/D)$, though empirical evidence in \citet{kearns2026networked} points to their analysis of this example being loose. In pursuit of a tighter bound in this range, we analyze this instance more carefully.

\begin{restatable}[Cyclic example lower bound]{theorem}{cyclicLowerRestatement}
\label{thm:kearns-cyclic-lower-bound}
Consider the cyclic example of \citet{kearns2026networked} for every integer
$k\ge2$ and every integer $1\le p \le k - 1$. Then the excess error of the final agent is at least $1/(48\sqrt p) = \Omega\left(1/\sqrt p\right)$. Equivalently, the excess error after $D=pk$ agents is at least
\begin{equation*}
  (1/48)\sqrt{k/D} = \Omega\left(\sqrt{M/D}\right).
\end{equation*}
\end{restatable}

\begin{proof}[Proof Sketch]
We change the basis of predictors to $z_1, \ldots, z_k$, since they are orthonormal. This reveals structure in the predictors. Let $Y - f_t$ be the residual of some agent $t$. We show that the residual of the agent at the end of pass $p$, denoted by $r^{(p)}$, has the following shape in the new basis:
\begin{equation*}
  r^{(p)} = (1 + 2S_p)^{-1} \left( S_p, S_p, \mu^{(p)}_1, \ldots, \mu^{(p)}_{p-1} \right),
\end{equation*}
where $S_p = \sum_i ( \mu^{(p)}_i )^2$, and the excess error of this agent is $E_p = S_p / (1 + 2S_p)$. Moreover, the \emph{tail} $\mu^{(p)}$ is a probability vector and we show an exact recursion for finding $\mu^{(p)}$ based on $\mu^{(p-1)}$.

The recursion on the tail $\mu^{(p)}$ matches a walk on the integers that starts at $1$, and at each step moves up by one and then falls back by a geometrically distributed amount, and is killed when it reaches $0$ or below. Encoding the steps of the walk as words over the alphabet $\{ (,) \}$ connects it to the Catalan numbers. We use the Catalan generating function to get a closed form for the survival probability of the walk, and show that it is $\Theta(1/\sqrt t)$ at time $t$. Conditioned on survival, the second moment of the walk grows only linearly with $t$. So after $p$ passes the squared mass $S_p$ of the tail is $\Omega(1/\sqrt p)$, and so the excess error $E_p$ is $\Omega(1/\sqrt p)$.
\end{proof}

Despite the challenges in analyzing the cyclic example, even the improved bound is still far from the constant upper bound in the range $D<M^2$. The example also gives no bound for $D\ge M^2$, so we move on and consider a new example that covers the range $D\ge M^2$ as well.

\begin{restatable}[Lower bound]{theorem}{longDepthRestatement}
\label{thm:long-depth}
For every $M\ge8$ and $D\ge M^2$, there is an $M$-covered path instance with an
exact global predictor whose coefficient $\ell_1$ norm is at most $3$,
each feature has second moment at most $2$, and whose path predictor $f_D$
satisfies
\begin{equation*}
  \E[(Y-f_D)^2]
  \ge
  \frac{1}{1280\pi^2}\frac{M^2}{D}.
\end{equation*}
\end{restatable}
\begin{proof}[Proof Sketch]
We construct an example designed to make information aggregation as slow as possible.

Fix $M,D$. We use three independent standard Gaussian variables $Z_0,Z_1,Z_2$, and introduce $M$ features $X_0, \dots, X_{M-1}$ where for $j=0,\dots,M-1$,
\begin{equation*}
  X_j = Z_1 + \rho(\cos(j\delta)Z_0 + \sin(j\delta)Z_2),
\end{equation*}
where $\delta = 2\pi/M$ and $\rho$ is a small radius chosen below. We then construct a path of length $D$ where the agents see features 
\begin{equation*}
X_0, X_1, \ldots, X_{M-1}, X_0, X_1, \ldots
\end{equation*}
in the same cyclic manner as in the cyclic example. We set the label as $Y=\rho Z_0$.

This way, the target signal is in a direction that no single feature reveals, and the feature seen by successive agents rotates by the small angle $\delta$. The residual of each agent is always orthogonal to the feature just seen, and the next feature points in almost the same direction, so each agent can remove only an $O(\rho^2\sin^2\delta)$ fraction of the remaining error. Choosing $\rho^2=\Theta(M^2/D)$ makes this fraction $\Theta(1/D)$, so a constant fraction of the starting error, itself of order $\rho^2$, survives all $D$ agents, leaving excess error $\Omega(\rho^2)=\Omega(M^2/D)$.

Though this instance has label variance only $\rho^2$, the full construction fixes the scale by adding a common feature $X_\star$, an independent standard Gaussian seen by every agent, and setting $Y=X_\star+\rho Z_0$. Every agent learns $X_\star$ at once, so the analysis is unchanged.
\end{proof}

Setting $D=M^2$ in the above theorem gives a constant lower bound for the agent at depth $M^2$. Since the excess error is non-increasing, this gives the same constant lower bound for all agents at depths $t\le M^2$.

\begin{restatable}[Constant error before quadratic depth]{corollary}{constantBeforeQuadraticRestatement}
\label{cor:constant-before-quadratic}
For every integer $M\ge8$, there is an $M$-covered path instance of depth
$M^2$ with an exact global predictor, coefficient $\ell_1$ norm at most
$3$, and feature second moments at most $2$, such that, if $E_t$ is the excess
error after the first $t$ agents, then
\begin{equation*}
  E_t\ge \frac{1}{1280\pi^2}
  \qquad
  (1\le t\le M^2).
\end{equation*}
\end{restatable}

Together with the upper bound, this gives the right order up to constants:
constant excess error can persist until depth $M^2$, and after that the rate is
$M^2/D$. 

Our construction's distribution depends on the depth $D$. We show that this cannot be avoided: under any fixed distribution on a finite set of features, the path predictor improves geometrically along every $M$-covered path, so no such fixed distribution can witness the $M^2/D$ lower bound at every depth.

\begin{restatable}[No fixed distribution for all depths]{theorem}{noFixedAllHorizonRestatement}
\label{thm:no-fixed-all-horizon}
Fix $M$ and a distribution $\cD$ on $(x_1,\ldots,x_d,Y)$. Assume
$x_1,\ldots,x_d,Y$ have bounded second moments. For every $c>0$, there is a depth $D_c$ such that for any DAG of agents on $\cD$, if $A_1\to\cdots\to A_D$ is an $M$-covered path in it with $D\ge D_c$, then
\begin{equation*}
  \cE_D<c\frac{M^2}{D},
\end{equation*}
where $\cE_D$ is the excess error of agent $A_D$.
\end{restatable}

To show this, we prove a geometric upper bound on the excess error for any fixed distribution.

\begin{restatable}[Fixed-distribution geometric convergence]{theorem}{fixedInstanceRestate}
\label{thm:fixed-distribution-geometric}
Fix $M$ and a distribution $\cD$ on $(x_1,\ldots,x_d,Y)$ with $d$ finite.
Assume $x_1,\ldots,x_d,Y$ have bounded second moments. There is a constant $q\in[0,1)$, depending only on $\cD$ and $M$, such that for any DAG of agents on $\cD$, if $A_1\to\cdots\to A_D$ is an $M$-covered path in it, then for any $1 \le s \le D-M$,
\begin{equation*}
  \cE_{s+M}\le q\cE_s,
\end{equation*}
where $\cE_t$ is the excess error of agent $A_t$. 
Consequently, $\cE_D\le \cE_1q^{\lfloor (D-1)/M\rfloor}$.
\end{restatable}

We prove this by considering a single path of $M$ agents that together see all raw features. For the drop in the excess error along this path, we give a factor $q\in[0,1)$ that depends only on the feature subsets seen by the agents and the distribution $\cD$, and not on the rest of the network. The proof of the theorem then applies this to every block of $M$ consecutive agents: since there are only finitely many possible tuples of feature subsets, the maximum of their factors is a single $q<1$ that works for every block.

This upper bound is also interesting in itself. It shows that any fixed distribution will eventually converge very quickly to the optimal solution.

\subsubsection{Classification}

We consider the binary classification protocol introduced by \citet{bateni2026networked}, which uses the same underlying network model as regression. There is a distribution $\cD$ over $(x_1, \ldots, x_d, Y)$ where $Y \in \{0,1\}$, and a DAG of $N$ agents $A_1, \ldots, A_N$. Each agent $A_i$ sees the features $x_{S_i}$ for a subset $S_i \subseteq [d]$, and the logits of all its parents in the DAG. The agent fits a linear logit $z_i$, a linear combination of $x_{S_i}$ and the parent logits, and then predicts $\Pr(Y = 1 \mid x)$ as $\sigma(z_i)$, where $\sigma$ is the sigmoid function. The agent chooses $z_i$ to minimize the binary cross-entropy loss. Instead of passing the probabilities, agents pass the logits $z_i$ to their successors. Let $\cL(z)$ denote the binary cross-entropy loss.

The proof of \Cref{thm:regression-improved-upper} uses only facts about least squares that have analogues in \citet{bateni2026networked}, under the same kind of bounded coefficient and second-moment assumptions. So the same induction from the proof of the regression theorem again gives the optimal upper bound.

\begin{restatable}[Improved classification upper bound]{theorem}{classificationUpperRestatement}
\label{thm:classification-improved-upper}
For every depth $D$ and every $M$-covered path, assume the global BCE
minimizer is $z_\star=\sum_{\ell=1}^dw_\ell x_\ell$, that
$\sum_\ell|w_\ell|\le B_\star$, that $\E[x_\ell^2]\le B_X^2$ for every
feature, and that the protocol minimizers are attained at finite coefficients. Then
\begin{equation*}
  \cL(z_D)-\cL(z_\star)
  \le
  C\frac{M^2}{D},
\end{equation*}
where $C$ depends only on $B_\star$ and $B_X$.
\end{restatable}

For lower bounds, we use the same examples as in \Cref{thm:kearns-cyclic-lower-bound,thm:long-depth}. To transform the regression label $Y \in \R$ to a label $Y^c \in \{0,1\}$, we use $\Pr(Y^c = 1 \mid x) = \sigma(Y)$. To analyze these examples for classification, one may restate the regression analysis using the classification analogues of the regression facts. We instead take a shortcut. These regression examples contain only Gaussian features and a Gaussian label. We can therefore directly use their already proven bounds for classification once we prove the transfer theorem below. In the special case where a single agent sees all the features that generate $Y$, the logistic and least-squares fits are known to align~\citep{brillinger1982generalized,erdogdu2016scaled}.

\begin{restatable}[Gaussian transfer to classification]{theorem}{classificationTransferRestatement}
\label{thm:classification-sequential-transfer}
Let $x_1, \dots, x_d$ be centered jointly Gaussian variables and
$G = w^T x$ be a linear combination of them. Let $Y\in\{0,1\}$ satisfy
$\Pr(Y=1\mid x) = \sigma(G)$. Assume that $\operatorname{Var}(G) \le B$.

Let $z_t$ be the logit predictor of agent $A_t$. Assume the same network is
used to predict $G$ with least-squares loss. Let $f_t$ be the predictor of
$A_t$, and let $E_t$ be the excess error in this network, i.e.,
$E_t = \norm{G - f_t}_2^2$.

Then, for all $t$, $z_t = c_t f_t$ for some $c_t \in [0,1]$, with $c_t > 0$
whenever $f_t \neq 0$.
Furthermore, for a constant $\kappa_B>0$ depending only on $B$,
\begin{equation*}
  \kappa_B E_t \le \cL(z_t)-\cL(G)\le\frac18 E_t.
\end{equation*}
\end{restatable}
\begin{proof}[Proof Sketch]
We start from a simpler case. We consider a single agent with inputs $u_1, \ldots, u_m$, and let $z$ be its BCE predictor of $Y$. Let $f$ be the least-squares predictor of $G$ on the same inputs. We show that $z = cf$ for some $c \in [0,1]$, with $c > 0$ whenever $f \neq 0$. This means that the two predictors are in the same direction. 

To prove this, we write $z=cf+r$, where $r$ lies in the span of the inputs and is orthogonal to $f$. The least-squares residual $G-f$ is orthogonal to that whole span, so $r$ is orthogonal to $G$, and orthogonal jointly Gaussian variables are independent, so $r$ is independent of $G$, $f$, and $Y$. By Jensen's inequality and the strict convexity of the BCE loss, removing such an independent component can only decrease the loss, so the minimizer lies on the line spanned by $f$.

All that remains is to locate $c$. Restricted to the line, the loss is convex in $c$, so it suffices to check the derivative at the endpoints: when $f\ne0$, it is negative at $c=0$, since $f$ is positively correlated with $Y$, and nonnegative at $c=1$. Hence $c\in(0,1]$.

Then consider a classification network with predictors $z_1, \ldots, z_N$ and its corresponding regression network with predictors $f_1, \ldots, f_N$. Having $z = cf$ for a single agent, we induct on the network according to its topological ordering, to show that the direction stays the same at all agents between the two networks. The first agent receives no predictions, so by the above, $z_1 = c_1 f_1$. For a later agent $A_i$, the induction hypothesis gives that its inputs span the same directions in the two networks, so $z_i = c_i f_i$.

We also show that there exists a constant $\kappa_B$ depending only on $B$ such that, whenever $\E[z^2] \le B$ and $\E[G^2] \le B$,
\begin{equation*}
  \kappa_B \E[(z - G)^2] \le \cL(z) - \cL(G) \le \frac18 \E[(z - G)^2].
\end{equation*}

Combining the fact that the two predictors point in the same direction with the above inequality, we derive the theorem's inequality directly.
\end{proof}

Applying this transfer theorem to
\Cref{thm:kearns-cyclic-lower-bound,thm:long-depth} gives the same lower-bound
picture for classification. The shallow constant lower bound is again a
corollary of the depth-dependent construction.

\begin{restatable}[Classification cyclic example lower bound]{theorem}{classificationCyclicRestatement}
\label{thm:classification-cyclic}
For every $k\ge2$, there is a $k$-covered Gaussian classification path instance
with true logit $G$ and Bernoulli labels with mean $\sigma(G)$ such that, at the
end of pass $1\le p\le k-1$,
\begin{equation*}
  \cL(z_{pk})-\cL(G)
  \ge
  \frac{\kappa_1}{48\sqrt p}.
\end{equation*}
Equivalently, if $D=pk$, then
\begin{equation*}
  \cL(z_D)-\cL(G)
  \ge
  \frac{\kappa_1}{48}\sqrt{\frac{k}{D}}.
\end{equation*}
\end{restatable}

\begin{restatable}[Classification depth-dependent lower bound]{theorem}{classificationLongDepthRestatement}
\label{thm:classification-long-depth}
For every $M\ge8$ and $D\ge M^2$, there is an $M$-covered Gaussian
classification path instance of depth $D$ with true logit $G$, Bernoulli labels
with mean $\sigma(G)$, and $G$ an exact linear combination of the features whose coefficient
$\ell_1$ norm is at most $3$, with $\E[x_\ell^2]\le2$ for every feature, and
with
\begin{equation*}
  \cL(z_D)-\cL(G)
  \ge
  c_{\mathrm{ld}}\frac{M^2}{D}
\end{equation*}
for a universal constant $c_{\mathrm{ld}}>0$.
\end{restatable}

\begin{restatable}[Classification constant loss before quadratic depth]{corollary}{classificationConstantRestatement}
\label{cor:classification-constant-before-quadratic}
There is a universal constant $c_{\mathrm{quad}}>0$ such that, for every
$M\ge8$, there is an $M$-covered Gaussian classification path instance of depth
$M^2$ with true logit $G$ an exact linear combination of the features whose coefficient $\ell_1$ norm is
at most $3$, with $\E[x_\ell^2]\le2$ for every feature, and with Bernoulli
labels with mean $\sigma(G)$ such that, for every $1\le t\le M^2$,
\begin{equation*}
  \cL(z_t)-\cL(G)\ge c_{\mathrm{quad}}.
\end{equation*}
\end{restatable}

Finally, as in the regression setting, the classification lower bound cannot
come from one fixed finite distribution, unless the distribution changes with the
target depth: the excess loss again contracts geometrically along every covered
path. The proof parallels the regression one, with the probability residual
$\sigma(z_t)-\sigma(z_\star)$ in place of $f^\star-f_t$. The one difference is
that the regression proof turned the excess error into a squared distance, and
cross-entropy is not a squared norm; but every logit on the path has loss at most
$\cL(0)=\log2$, and on this bounded set the loss is strongly convex, which
again makes the excess loss comparable to the squared distance between
logits.

\begin{restatable}[No fixed classification distribution for all depths]{theorem}{classificationNoFixedRestatement}
\label{thm:no-fixed-classification-distribution}
Fix $M$ and a distribution $\cD$ on $(x_1,\ldots,x_d,Y)$ with $d$ finite,
bounded second moments for the raw features, and $Y\in\{0,1\}$. Assume the BCE
minimum over the raw-feature span is attained, and call a minimizer
$z_\star$. For every $c>0$ there is a $D_c$ such that for any DAG of agents on
$\cD$, if
$A_1\to\cdots\to A_D$ is a path whose every $M$ consecutive raw-feature spans
sum to the full raw-feature span and $D\ge D_c$, then
\begin{equation*}
  \cL(z_D)-\cL(z_\star)
  <
  c\frac{M^2}{D}.
\end{equation*}
Thus one fixed finite classification distribution with an attained minimizer
cannot witness an $M^2/D$ lower bound for all depths.
\end{restatable}

\subsection{Additional Related Work}

The model we study is due to \citet{kearns2026networked}: the DAG formulation,
the $M$-coverage condition, and the first depth-based regression bounds are all
theirs, and \citet{bateni2026networked} adapted the protocol to binary
classification with agents passing logits rather than predictions. Working in
the same model without changes, we give matching upper and lower bounds of order
$M^2/D$ on covered paths, and we show that no single distribution can realize the
lower bound at every depth at once.

\paragraph{Opinion dynamics and social learning.}
The older backdrop is social learning in networks. In the model of
\citet{degroot1974reaching}, an agent starts with a belief about a common state
and at each step resets it to a weighted average of its neighbors' beliefs, a
fixed rule with no inference behind it, and \citet{golub2010naive} give
conditions under which the resulting consensus is correct. In sequential
variants, agents instead act one at a time on their own signal and the actions
they have seen, and \citet{banerjee1992simple} and \citet{bikhchandani1992theory}
show how this leads to herding, where a few early movers fix the outcome and
later signals go unused. \citet{gale2003bayesian} analyze Bayesian agents
learning from their neighbors' actions, \citet{mossel2016efficient} study agents
who exchange Gaussian estimates over many rounds, and \citet{mossel2018social}
characterize when such exchange reaches an equilibrium that aggregates everyone's
information. Throughout this line there is one hidden state observed through
noise, and the question is whether it is recovered in the limit. Ours is a
different problem. There is no single hidden state: the label comes from a joint
distribution over many correlated features, and we compare against the best
linear predictor that sees all of them. The path is used only once, so what
limits the final agent is its depth, not the number of rounds.

\paragraph{Reusing predictions as features.}
Predictions can also be fed forward as features. Stacked generalization does
exactly this, training a second model on the first model's outputs
\citep{wolpert1992stacked}. Multicalibration and multiaccuracy
\citep{hebertjohnson2018multicalibration,kim2019multiaccuracy}, and later
outcome indistinguishability \citep{gopalan2023loss}, ask when one predictor can
stand in for many losses or downstream decisions at once, and
\citet{noarov2026optimal} show a deterministic predictor already suffices, at
optimal sample complexity. In all of this a downstream learner still sees the
whole prediction next to the raw features. Our agents get far less to work with,
just their own features and one number from upstream, and the whole question is
how much of the label survives that number being recomputed at every hop.

\paragraph{Agreement protocols.}
Prediction-passing is also the medium of the agreement literature. Aumann's
theorem \citep{aumann1976agreeing} already says that two Bayesians who keep
trading posteriors cannot end up disagreeing, and \citet{collina2025tractable}
and \citet{collina2026collaborative} reach the same conclusion without full
Bayesian agents, needing only calibration conditions a learner can enforce,
after which a few exchanges bring both parties to a shared prediction that pools
what each of them knew. \citet{eaton2026model} come at it from the other side and
bound how far two independently trained models can disagree in the first place.
Our upper bounds lean on multiaccuracy and self-orthogonality, cousins of those
calibration conditions that ordinary least squares happens to satisfy for free.
The one structural difference is traffic. Agreement runs both ways and many
times over, each side updating against the other, whereas on our path a
prediction is made once and passed on, and no agent ever answers back.

\paragraph{Vertical federated learning and distributed optimization.}
Splitting features across parties is the setting of vertical federated learning
and split learning. \citet{yang2019federated} survey the area. SecureBoost
\citep{cheng2021secureboost} and split learning \citep{vepakomma2018split} train
a single shared model over such a split without exposing raw features, by
exchanging gradients, activations, or masked statistics over many rounds. The
related literature on gossip and distributed optimization
\citep{boyd2006randomized,zhang2013communication,chen2017communication} studies
how the communication graph slows a joint optimization. In our protocol there is
no shared model to train. Each agent fits its own model once and passes on a
single prediction or logit, and the error reflects how much is lost by reducing
each agent's output to that single number at every step.

\section{Preliminaries}
\label{sec:preliminaries}

In this section we restate the model formally and introduce the necessary notation. We write $[k]$ for the set $\{1,2,\dots,k\}$. All expectations are over the distribution $\mathcal{D}$.

\subsection{Model, Regression Protocol, and MSE Benchmark}

We start with the regression model of \citet{kearns2026networked}. There
are agents $\cA=\{A_1,\ldots,A_N\}$ in a directed acyclic graph
$G=(\cA,E)$. An edge $A_j\to A_i$ means that $A_i$ receives the prediction made
by $A_j$. We write $\Pa(i) = \{A_j : (A_j\to A_i) \in E\}$ for the parents of $A_i$. Agents learn in a
topological order.

The population distribution $\mathcal{D}$ is over feature-label pairs $(x, Y)$, where $x\in\R^d$ and $Y\in\R$. Agent $A_i$ sees only the
coordinates $x_{S_i}$, for $S_i\subseteq[d]$. Agent $A_i$ also sees the parent predictions $f_j(x)$ for
$A_j\in\Pa(i)$ and chooses the best linear prediction of $Y$ from these inputs.
Thus
\begin{equation}
  \label{eq:regression-protocol}
  \begin{aligned}
    f_i(x)
    &=
    w_i^\top x_{S_i}
    +
    \sum_{A_j\in\Pa(i)}v_{ij}f_j(x),\\
    \{w_i,v_{ij}\}
    &\in
    \arg\min
    \E[(f_i(x)-Y)^2].
  \end{aligned}
\end{equation}

Following \citet[Definitions~2.1 and~2.2]{kearns2026networked}, the error
of a predictor is $\MSE(f)=\E[(f(x)-Y)^2]$, and
$\|Z\|_2=\sqrt{\E[Z^2]}$ for a random variable $Z$. The
global predictor is $f^\star(x)=(w^\star)^\top x$, where
$w^\star$ minimizes $\E[((w^\top x)-Y)^2]$ over all $w\in\R^d$.

On a path of depth $D$, we write the agents as $A_1\to A_2\to\cdots\to A_D$. Each agent can
always keep the incoming prediction, so the error along the path is non-increasing, as shown below in \Cref{lem:prelim-mse-decomposition}. For this path, we use the following definition of \citet{kearns2026networked}.

\begin{definition}[$M$-covered path]
\label{def:m-coverage}
A path $A_1\to A_2\to\cdots\to A_D$ is $M$-covered if every block of $M$ consecutive agents collectively sees all features in $[d]$. In other words, $\bigcup_{j=i}^{i+M-1} S_{j}=[d]$ for every $i \le D-M+1$.
\end{definition}

\subsection{Regression Facts Used Later}

We import the following two lemmas without proof from \citet{kearns2026networked}.

\begin{lemma}[\citet{kearns2026networked}, Lemma~3.1 and Corollary~3.2]
\label{lem:prelim-ls-orthogonality}
If $f$ is the least-squares predictor from inputs $u_1,\ldots,u_m$, then
$\E[u_r(f-Y)]=0$ for each input $u_r$. Since $f$ is a linear combination of its
inputs, also $\E[f(f-Y)]=0$.
\end{lemma}

These are the multiaccuracy and self-orthogonality conditions from
\citet[Definitions~2.3 and~2.4]{kearns2026networked}.

\begin{lemma}[\citet{kearns2026networked}, Lemmas~3.3, 3.7, and~3.8]
\label{lem:prelim-mse-decomposition}
For any predictors $f$ and $g$,
\begin{equation}
  \label{eq:mse-decomposition}
  \MSE(f)
  =
  \MSE(g)
  -2\E[g(f-Y)]
  +2\E[f(f-Y)]
  -\E[(f-g)^2].
\end{equation}
If $A_j \in \Pa(i)$, then
$\MSE(f_i)\le\MSE(f_j)$ and
\begin{equation}
  \label{eq:path-improvement}
  \begin{aligned}
    \E[(f_i-f_j)^2]
    &=
    \MSE(f_j)-\MSE(f_i).
  \end{aligned}
\end{equation}
\end{lemma}

The next lemma restates the least squares problem in a more useful way. Suppose the inputs are
$u_1,\ldots,u_m$, and their span is the set of predictors that can be formed from them:
$\{\sum_{r=1}^m a_ru_r : a_1,\ldots,a_m\in\R\}$. We say two random variables $U$ and $W$ are orthogonal when $\E[UW]=0$. Thus the equations in \Cref{lem:prelim-ls-orthogonality} say that
least squares leaves a residual $Y-f$ that is orthogonal to every predictor in the span.

\begin{lemma}[Least squares is projection]
\label{lem:prelim-ls-projection}
Let $Y,u_1,\ldots,u_m$ be random variables with finite second moments, and let
$V=\operatorname{span}\{u_1,\ldots,u_m\}$. If $f$ is the least-squares predictor
from inputs $u_1,\ldots,u_m$, then $f$ is the orthogonal projection of $Y$ onto
$V$: up to almost sure equality, it is the unique random variable in $V$ such
that $Y-f$ is orthogonal to every element of $V$.
\end{lemma}

\begin{proof}
\Cref{lem:prelim-ls-orthogonality} gives $\E[v(Y-f)]=0$ for every $v\in V$ by
linearity. For uniqueness, suppose $f_1,f_2\in V$ both have residuals
orthogonal to $V$. Since $f_1-f_2\in V$, substituting $v=f_1-f_2$ into the
orthogonality condition for $f_1$ and for $f_2$ gives
\begin{equation*}
  \E[(Y-f_1)(f_1-f_2)]=0,
  \qquad
  \E[(Y-f_2)(f_1-f_2)]=0.
\end{equation*}
Subtracting the two equations we get
\begin{equation*}
  \E[(f_1-f_2)^2]=0.
\end{equation*}
Hence $f_1=f_2$ almost surely.
\end{proof}

\subsection{Classification Setup}

For classification, $Y\in\{0,1\}$. A logit $z$ gives the probability
$p=\sigma(z)$, where $\sigma(t)=1/(1+e^{-t})$ is the sigmoid function. The
population binary cross-entropy (BCE) loss is
\begin{equation}
  \cL(z)=\E[\log(1+e^{z(x)})-Yz(x)].
\end{equation}
We also write $\cL(p)$ when $p=\sigma(z)$ and define
$\phi(z) = \log(1+e^z)$. Note that $\phi(z)$ satisfies
$\phi'(z) = \sigma(z)$.

The classification model follows the logit-passing setup of
\citet{bateni2026networked} and is similar to the regression model. Agent $A_i$ receives parent logits
$z_j(x)$, not parent probabilities. It chooses the coefficients in
\begin{equation}
  \label{eq:classification-protocol}
  \begin{aligned}
    z_i(x)
    &=
    w_i^\top x_{S_i}
    +
    \sum_{A_j\in\Pa(i)}v_{ij}z_j(x),
    \\
    p_i(x)
    &=
    \sigma(z_i(x)),
    \qquad
    \{w_i,v_{ij}\}
    \in
    \arg\min \cL(z_i).
  \end{aligned}
\end{equation}

We assume the protocol minimizers are attained at finite
coefficients. The global BCE minimizer is written as
$z_\star(x)=\sum_{\ell=1}^dw_\ell x_\ell$ and
$p_\star=\sigma(z_\star)$.

\subsection{Classification Facts Used Later}

We import the following facts from \citet{bateni2026networked}.

\begin{lemma}[\citet{bateni2026networked}, Lemma~3.1]
\label{lem:prelim-bce-orthogonality}
If $z$ is a finite BCE minimizer over linear inputs $u_1,\ldots,u_m$, then each
input has zero correlation with the BCE residual. In particular,
$\E[z(\sigma(z)-Y)]=0$.
\end{lemma}

\begin{lemma}[\citet{bateni2026networked}, Definition~3.2 and Lemmas~3.3--3.4]
\label{lem:prelim-bce-kl}
For two probability predictors $p$ and $q$, define
\begin{equation}
  \begin{aligned}
    D(p\|q)
    &=
    \E\left[
      p\log\frac{p}{q}
      +(1-p)\log\frac{1-p}{1-q}
    \right].
  \end{aligned}
\end{equation}
Then,
\begin{equation}
  \label{eq:bce-kl-identities}
  \begin{aligned}
    \cL(q)
    &=
    \cL(p^\star)+D(p^\star\|q),\\
    D(p\|q)
    &\ge
    2\E[(p-q)^2].
  \end{aligned}
\end{equation}
Here $p^\star$ is the BCE minimizer on the same inputs as $q$. If $A_j \in \Pa(i)$, then
$\cL(z_j)-\cL(z_i)=D(p_i\|p_j) \ge 0$.
\end{lemma}

\begin{lemma}[\citet{bateni2026networked}, Lemma~3.5]
\label{lem:prelim-classification-comparator}
If $z$ minimizes BCE over its current inputs and $z_g$ is any other logit, then
\begin{equation}
  \label{eq:classification-comparator}
  \begin{aligned}
    \cL(z)
    &\le
    \cL(z_g)+\left|\E[(\sigma(z)-Y)z_g]\right|.
  \end{aligned}
\end{equation}
\end{lemma}

This is the classification replacement for the least-squares residual
comparison.

The BCE loss is convex, a fact we rely on repeatedly when locating minimizers,
so we record it once here.

\begin{lemma}
\label{lem:prelim-bce-convexity}
$\phi(z) = \log(1+e^z)$ is strictly convex. The BCE loss $\cL(z) = \E[\phi(z) - Yz]$ is convex.
\end{lemma}

\begin{proof}
Since $\phi''(z) = \sigma'(z) = \sigma(z)(1 - \sigma(z)) > 0$ for every $z$, $\phi$ is strictly convex.

For each label $Y$, the map $z\mapsto\phi(z)-Yz$ is convex, and averaging over the distribution
shows that the BCE loss $\cL(z)=\E[\phi(z)-Yz]$ is convex.
\end{proof}

\section{Improved Regression Upper Bound}
\label{sec:regression-upper-bound}

For the rest of this section, we assume the following conditions from \cite{kearns2026networked}:
\begin{itemize}
\item The global predictor
$f^\star(x)=\sum_{\ell=1}^d w^\star_\ell x_\ell$ satisfies
$\sum_\ell |w^\star_\ell|\le A^\star$.
\item Each feature satisfies $\E[x_\ell^2]\le M_X^2$.
\end{itemize}

Consider an $M$-covered path of agents $A_1 \to A_2 \to \cdots \to A_n$. By \Cref{lem:prelim-mse-decomposition}, the errors are non-increasing along the path, meaning that $\MSE(f_1) \ge \MSE(f_2) \ge \cdots \ge \MSE(f_n)$. Now consider a block of $M$ agents $A_i \to A_{i+1} \to \cdots \to A_{i+M-1}$. \citet{kearns2026networked} show that if the error does not decrease significantly over the block, then the excess error of the last agent in the block is small. We formalize this in the following lemma, which is analogous to Theorem 3.9 in \cite{kearns2026networked}. We provide its proof for completeness.

\begin{lemma}
\label{lem:regression-block-residual}
Consider any path $A_1 \to A_2 \to \dots \to A_n$ and a block of $k$ consecutive agents indexed by $[a+1, b]$ that sees every raw feature at least
once, where $b=a+k$.
If $\varepsilon\ge\MSE(f_a)-\MSE(f_b)$, then
\begin{equation*}
  \MSE(f_b)-\MSE(f^\star)
  \le
  2A^\star M_X\sqrt{k\varepsilon}.
\end{equation*}
\end{lemma}

\begin{proof}
Fix a feature $x_\ell$, and let agent $A_j$ in the block see it. Least-squares orthogonality from \Cref{lem:prelim-ls-orthogonality} gives $\E[x_\ell(f_j-Y)]=0$. By
\Cref{eq:path-improvement} in \Cref{lem:prelim-mse-decomposition},
\begin{equation*}
  \sum_{i=a+1}^b \|f_i-f_{i-1}\|_2^2
  =
  \MSE(f_a)-\MSE(f_b)
  \le
  \varepsilon.
\end{equation*}
Since agent $A_j$ sees $x_\ell$, its orthogonality $\E[x_\ell(f_j-Y)]=0$ lets us
replace the target $Y$ with $f_j$ in the following equation:
\begin{equation*}
  \E[x_\ell(f_b-Y)]
  =
  \E[x_\ell(f_b-f_j)]
  +
  \E[x_\ell(f_j-Y)]
  =
  \E[x_\ell(f_b-f_j)].
\end{equation*}
Write $f_b - f_j$ as $f_b-f_j=\sum_{i=j+1}^b(f_i-f_{i-1})$. The sum has at most
$k$ terms, since $j\ge a+1$ and $b=a+k$. The triangle inequality followed by Cauchy--Schwarz across these terms gives
\begin{equation*}
  \|f_b-f_j\|_2
  \le
  \sum_{i=j+1}^b\|f_i-f_{i-1}\|_2
  \le
  \sqrt{k}\left(\sum_{i=j+1}^b\|f_i-f_{i-1}\|_2^2\right)^{1/2}
  \le
  \sqrt{k\varepsilon},
\end{equation*}
where the last step uses the bound $\sum_{i=a+1}^b\|f_i-f_{i-1}\|_2^2\le\varepsilon$
from above. A final Cauchy--Schwarz over the feature, with $\E[x_\ell^2]\le M_X^2$,
then yields
\begin{equation*}
  |\E[x_\ell(f_b-Y)]|
  =
  |\E[x_\ell(f_b-f_j)]|
  \le
  \sqrt{\E[x_\ell^2]}\,\|f_b-f_j\|_2
  \le
  M_X\sqrt{k\varepsilon}.
\end{equation*}
The bound just derived holds for every feature $\ell$, since the block sees each
raw feature at least once. To turn these per-feature bounds into an MSE bound,
apply \Cref{eq:mse-decomposition} from \Cref{lem:prelim-mse-decomposition} with $f=f_b$ and $g=f^\star$:
\begin{equation*}
  \MSE(f_b)-\MSE(f^\star)
  =
  -2\E[f^\star(f_b-Y)]
  +2\E[f_b(f_b-Y)]
  -\E[(f_b-f^\star)^2].
\end{equation*}
Self-orthogonality of $f_b$ from \Cref{lem:prelim-ls-orthogonality} makes the
middle term vanish, $\E[f_b(f_b-Y)]=0$, and the last term is nonpositive, so
$\MSE(f_b)-\MSE(f^\star)\le2|\E[f^\star(f_b-Y)]|$. Finally, expand
$f^\star=\sum_\ell w^\star_\ell x_\ell$ and use the per-feature bound together with
$\sum_\ell|w^\star_\ell|\le A^\star$,
\begin{equation*}
  |\E[f^\star(f_b-Y)]|
  \le
  \sum_\ell|w^\star_\ell|\,|\E[x_\ell(f_b-Y)]|
  \le
  A^\star M_X\sqrt{k\varepsilon},
\end{equation*}
which gives $\MSE(f_b)-\MSE(f^\star)\le2A^\star M_X\sqrt{k\varepsilon}$.
\end{proof}

We now state the following lemma. This lemma is what allows us to get a bound of $O(M^2/D)$ instead of the $O(M/\sqrt{D})$ bound of \citet{kearns2026networked}. It states that in a long enough path, if an agent near the start has excess error $\delta$, then after another $L$ agents, the excess error is at most $O(\sqrt{\delta/L})$.

\begin{lemma}[A good suffix from a good prefix]
\label{lem:regression-suffix-improvement}
Take an $M$-covered path $A_1 \to A_2 \to \dots \to A_n$ and an agent $A_s$ on it with
$\MSE(f_s) - \MSE(f^\star)\le\delta$. Let $L = n-s$ and consider the suffix $A_{s+1}\to\dots\to A_{s+L}$. If $L \ge 2M$, then
\begin{equation*}
  \MSE(f_{n})-\MSE(f^\star)
  \le
  2\sqrt2\,A^\star M_X\,M\sqrt{\frac{\delta}{L}}.
\end{equation*}
\end{lemma}

\begin{proof}
Split the suffix into $K=\lfloor L/M\rfloor$ full blocks. Since $L\ge2M$,
$K\ge L/(2M)$. The total MSE drop over these blocks is at most
$\MSE(f_s)-\MSE(f^\star)\le\delta$, so by the pigeonhole principle, some block has drop
$\varepsilon\le\delta/K\le2M\delta/L$. Applying
\Cref{lem:regression-block-residual} on this block with $k=M$ gives, at the end $q$ of that
block,
\begin{equation*}
  \MSE(f_q)-\MSE(f^\star)
  \le
  2A^\star M_X\sqrt{\frac{2M^2\delta}{L}}.
\end{equation*}
By \Cref{lem:prelim-mse-decomposition}, the MSE is non-increasing along the path, so $\MSE(f_{n})\le\MSE(f_q)$.
\end{proof}

\citet{kearns2026networked} prove their $O(M/\sqrt{D})$ by considering the drop in the excess error along a path of depth $D$. They argue that this drop is bounded by the first agent's excess error $\MSE(f_1)-\MSE(f^\star)$, and then partition the path into $K = \lfloor D / M \rfloor$ blocks of size $M$. By the pigeonhole principle, there must be a block that has a drop of at most $(\MSE(f_1)-\MSE(f^\star)) / K$. They then apply \Cref{lem:regression-block-residual}.

We take a different approach by inducting on $D$. We break a path of depth $D$ into a prefix of length $s = \lfloor D/2 \rfloor$ and a suffix of length $D - s$. We use the bound given by induction hypothesis, and apply \Cref{lem:regression-suffix-improvement} on the suffix.
The following restated theorem proves this formally.

\regressionUpperRestatement*

\begin{proof}
Set $c_\star=2\sqrt2\,A^\star M_X$, the constant of
\Cref{lem:regression-suffix-improvement}, and set $C=6c_\star^2$.

We first bound the excess error of any agent. The zero predictor is feasible for every agent, so
$\MSE(f_i)\le\MSE(0)$. Applying \Cref{eq:mse-decomposition} with
$f=f^\star$ and $g=0$, and then using the self-orthogonality of $f^\star$ from
\Cref{lem:prelim-ls-orthogonality}, gives
\begin{equation*}
  \MSE(0)-\MSE(f^\star)
  =
  \|f^\star\|_2^2
  \le
  \Bigl(\sum_\ell |w^\star_\ell|\,\|x_\ell\|_2\Bigr)^2
  \le
  (A^\star M_X)^2.
\end{equation*}
It follows that every agent satisfies
$\MSE(f_i)-\MSE(f^\star)\le(A^\star M_X)^2$.

We now prove that $\MSE(f_D)-\MSE(f^\star)\le CM^2/D$ for every $M$-covered
path of depth $D$, arguing by induction on $D$. If $D<4M$, then
$M^2/D>M/4\ge1/4$, and hence
\begin{equation*}
  \MSE(f_D)-\MSE(f^\star)
  \le
  (A^\star M_X)^2
  \le
  \frac{C}{4}
  \le
  C\frac{M^2}{D}.
\end{equation*}

Now suppose $D\ge4M$, and split the path at $s=\lfloor D/2\rfloor$ into a prefix
$A_1\to\dots\to A_s$ and a suffix $A_{s+1}\to\dots\to A_D$ of length $L=D-s$.
Both are $M$-covered, being sub-paths of an $M$-covered path. We apply the induction hypothesis on the prefix. Since
$s=\lfloor D/2\rfloor\ge(D-1)/2\ge D/3$ (using $D\ge3$),
\begin{equation*}
  \MSE(f_s)-\MSE(f^\star)
  \le
  \frac{CM^2}{s}
  \le
  \frac{3CM^2}{D}
  =:\delta.
\end{equation*}
The suffix has length $L=D-s\ge D/2\ge2M$, so
\Cref{lem:regression-suffix-improvement} applies with this $\delta$. Substituting
$\delta=3CM^2/D$ and $L\ge D/2$ then gives
\begin{equation*}
  \MSE(f_D)-\MSE(f^\star)
  \le
  c_\star M\sqrt{\frac{\delta}{L}}
  \le
  c_\star M\sqrt{\frac{3CM^2/D}{D/2}}
  =
  c_\star\sqrt{6C}\,\frac{M^2}{D}.
\end{equation*}
Finally, $C=6c_\star^2$ so $c_\star\sqrt{6C}=C$. Thus the right-hand side is $CM^2/D$, which completes the induction.
\end{proof}

\section{The Cyclic Example of Kearns, Roth, and Ryu}
\label{sec:kearns-cyclic-example}

\citet[Definitions~5.1 and~5.2]{kearns2026networked} used a path example to
show that depth is necessary, and showed its excess error is $\Omega(M/D)$. We
revisit the same example and show that it is harder than their
analysis found: its excess error is in fact $\Omega(\sqrt{M/D})$. 

For a fixed $k$, the example constructs $k$ features $X_1, \dots, X_k$ and a label. It then constructs a path network where each agent sees one raw feature, and in a cyclic manner: agents $A_1 \to A_2 \to \cdots$ see features $X_1, X_2, \dots, X_k, X_1, X_2, \dots, X_k, \dots$. Let a \emph{pass} be a single repetition of the features. 

Our analysis begins by giving the exact recursive formula for the predictor of the agent at the end of each pass in \Cref{subsec:residual-shape}. We then show that this recursion is similar to a recursion induced by a walk over integers, and analyze the second moment of the variables in the walk recursion. To do this, in \Cref{subsec:killed-walk}, we show a connection between the walk recursion and word over the parentheses alphabet $\{ (, ) \}$. This reveals a connection to the Catalan numbers. In \Cref{subsec:catalan}, we use a general form of the Catalan generating function to close the arguments.

Although the analysis proves to be challenging, we later show that this example is not the hardest possible one. The construction only reaches depth $D<M^2$, and even in this regime, as shown in \Cref{sec:depth-dependent-obstruction}, there is a different example with constant lower bound.

\subsection{The Path}

Fix $k\ge2$. Let $z_1,\ldots,z_k$ be independent standard Gaussians, let
$Y=z_k$, and set $X_1=z_1$ and $X_i=z_i-z_{i-1}$ for $2\le i\le k$. The network is a single path
that sees the features in the repeating cyclic order
$X_1,\ldots,X_k,X_1,\ldots,X_k,\ldots$, and one pass is one full cycle
of $X_1, \ldots, X_k$. Let $\widehat y_p$ be the prediction of the last agent
in pass $p$, the one that sees $X_k$, and let $R_p=Y-\widehat y_p$ be its
residual and $E_p=\E[R_p^2]$.

Every block of $k$ agents sees all features, thus the path is $k$-covered or $M$-covered with $M=k$. The features telescope to $X_1+\cdots+X_k=z_k=Y$, so the global predictor has zero error. Hence the error $E_p$ is exactly the excess error after $D=pk$ agents, which is what we lower bound. 

In the rest of this section, we will prove \Cref{thm:kearns-cyclic-lower-bound} restated below.

\cyclicLowerRestatement*

\subsection{Residual Shape and Tail Update}
\label{subsec:residual-shape}

We prove the theorem by tracking the residual $R_p$ from pass to pass, and we
first set up coordinates for it. Because $z_1,\ldots,z_k$ are independent
standard Gaussians, they are orthonormal, with $\E[z_az_b]=1$ when $a=b$ and $0$
otherwise. So we can change the basis for every predictor and residual, from $X_1, \dots, X_k$ to $z_1, \dots, z_k$. We index the $z$-coefficients of a vector by \emph{position} $j\ge0$: position $j$ holds $z_{k-j}$, so position $0$ holds $z_k=Y$. We write $e_j$
for the coefficient vector of $z_{k-j}$.

After $p$ passes the residual involves only $z_{k-p},\ldots,z_k$, that is
positions $0,\ldots,p$, so we may write $R_p=\sum_{j=0}^p r_j^{(p)}z_{k-j}$ with
coefficient vector $r^{(p)}=(r_0^{(p)},\ldots,r_p^{(p)})$. This fact is due to
\citet[Lemma~5.5]{kearns2026networked}; it is the first part of the lemma below,
whose proof we restate for completeness, and the second part strengthens it by describing the structure of the residual.

\begin{restatable}[End-of-pass shape]{lemma}{lemKearnsEndPassShape}
\label{lem:kearns-end-pass-shape}
For $1\le p\le k-1$, the residual has the form above and
\begin{align*}
  \sum_{j=0}^p r_j^{(p)}=1,\quad
  E_p=\sum_{j=0}^p (r_j^{(p)})^2=r_0^{(p)},\quad
  r_0^{(p)}=r_1^{(p)}.
\end{align*}
\end{restatable}

The proof of the above lemma is in \Cref{sec:omitted-proofs}.

In the next lemma, we analyze what each agent in the pass does to the residual locally. The lemma after that aggregates these changes across a pass.

\begin{restatable}[One-agent update]{lemma}{lemKearnsOneAgentUpdate}
\label{lem:kearns-one-agent-update}
Suppose an agent receives residual coefficient vector $b$, so its incoming
prediction has coefficient vector $e_0-b$. Let $b^+$ be the outgoing residual
coefficient vector.

If the feature seen by the agent is orthogonal to both $Y$ and the incoming
prediction, then $b^+=b$.

If the feature is $g_j=e_{j-1}-e_j$ with $j\ge2$, then for some scalar
$\alpha$,
\begin{equation*}
  b^+_{j-1}=b^+_j=\frac{\alpha}{2}(b_{j-1}+b_j),
  \qquad
  b^+_m=\alpha b_m
  \quad (m\ge1,\ m\notin\{j-1,j\}).
\end{equation*}
If the feature is $g_1=e_0-e_1$, then for some scalar $\alpha$,
\begin{equation*}
  b^+_0=b^+_1=\frac12\bigl[(1-\alpha)+\alpha(b_0+b_1)\bigr],
  \qquad
  b^+_m=\alpha b_m
  \quad (m\ge2).
\end{equation*}
\end{restatable}

\begin{restatable}[Tail shape]{lemma}{lemKearnsTailShape}
\label{lem:kearns-tail-shape}
Define probability vectors $\mu^{(p)}$ recursively as follows. Start with
$\mu^{(2)}=(1)$. For each $3\le p\le k-1$, after
$\mu^{(p-1)}=(\mu_1^{(p-1)},\ldots,\mu_{p-2}^{(p-1)})$ has been defined, set
$S_{p-1}=\sum_i(\mu_i^{(p-1)})^2$. Initialize
$u^{(p)}=(u_1^{(p)},\ldots,u_{p-1}^{(p)})$ to the zero vector, where
$u_i^{(p)}$ is the unnormalized mass assigned to position $i+1$. Add
$S_{p-1}/2$ to $u_1^{(p)}$.
For every $1\le m\le p-2$, the mass $\mu_m^{(p-1)}$ at position $m+1$ adds
$2^{-(m+2-i)}\mu_m^{(p-1)}$ to $u_i^{(p)}$ for each $1\le i\le m+1$.
Normalize:
$\mu^{(p)}=u^{(p)}/\sum_i u_i^{(p)}$. Finally set
$S_p=\sum_i(\mu_i^{(p)})^2$ for every $2\le p\le k-1$. Then for every
$2\le p\le k-1$,
\begin{equation*}
\begin{gathered}
  r^{(p)}
  =
  (1+2S_p)^{-1}
  (S_p,S_p,\mu_1^{(p)},\mu_2^{(p)},\ldots,\mu_{p-1}^{(p)}),
  \\
  E_p=\frac{S_p}{1+2S_p}.
\end{gathered}
\end{equation*}
\end{restatable}

The proofs of the above lemmas are in \Cref{sec:omitted-proofs}.

To prove a lower bound on $E_p$, we will use the following lemma.

\begin{lemma}[Second moment forces squared mass]
\label{lem:kearns-squared-mass}
If $\mu$ is a probability vector on the positive integers and
$m=\sum_i i^2\mu_i$, then $\sum_i\mu_i^2\ge3/(16\sqrt m)$.
\end{lemma}

\begin{proof}
Let $L=\lceil2\sqrt m\rceil$. Since the support is positive, $m\ge1$ and
$L\le3\sqrt m$. Markov's inequality gives
$\sum_{i>L}\mu_i\le m/L^2\le1/4$, so the first $L$ coordinates carry mass at
least $3/4$. By Cauchy--Schwarz,
$\sum_i\mu_i^2\ge(3/4)^2/L\ge3/(16\sqrt m)$.
\end{proof}

With the above lemma in mind, all that is left is to show an upper bound on the second moment of $\mu^{(p)}$. We state this as the following lemma.

\begin{lemma}
\label{lem:kearns-mu-second-moment}
For $2\le p \le k-1$,
\begin{equation*}
  \sum_i i^2\mu_i^{(p)} \le 9p.
\end{equation*}
\end{lemma}

We will show this lemma later in the next subsection. Having the above lemma, we can now directly show \Cref{thm:kearns-cyclic-lower-bound} through \Cref{lem:kearns-squared-mass}.

\begin{proof}[Proof of \Cref{thm:kearns-cyclic-lower-bound}]
For $p=1$ the least-squares predictor of $z_k$ from $X_k=z_k-z_{k-1}$ leaves residual
$(z_k+z_{k-1})/2$, so $E_1=1/2\ge1/48$.

For $p\ge2$, \Cref{lem:kearns-mu-second-moment} gives $\sum_i i^2\mu_i^{(p)} \le 9p$. Thus \Cref{lem:kearns-squared-mass} with $m\le9p$ gives
$S_p\ge3/(16\sqrt{9p})=1/(16\sqrt p)$. Since $\mu^{(p)}$ is a probability
vector, $S_p=\sum_i(\mu_i^{(p)})^2\le\sum_i\mu_i^{(p)}=1$, hence
$1+2S_p\le3$, and \Cref{lem:kearns-tail-shape} gives
\begin{equation*}
E_p=\frac{S_p}{1+2S_p}\ge\frac{S_p}{3}\ge\frac{1}{48\sqrt p}. \qedhere
\end{equation*}
\end{proof}

\subsection{The Killed Walk}
\label{subsec:killed-walk}

The update rule for $\mu^{(p-1)}$ to $\mu^{(p)}$ in \Cref{lem:kearns-tail-shape} can be described by tracking one unit of mass in a walk over the integers.

The unit mass starts at $W_0=1$. To take one step, it first increases this
integer by one and then subtracts a random amount:
\begin{equation*}
  W_{t+1}=W_t+1-G_{t+1},
\end{equation*}
where $\Pr(G_{t+1}=\ell)=2^{-(\ell+1)}$ for $\ell\ge0$, and
$G_1,G_2,\ldots$ are independent.

\begin{lemma}
\label{lem:kearns-killed-walk-transition}
Suppose $W_t=j>0$. Then
$\Pr(W_{t+1}=i\mid W_t=j)=2^{-(j+2-i)}$ for $1\le i\le j+1$, and
$\Pr(W_{t+1}\le0\mid W_t=j)=\sum_{\ell\ge j+1}2^{-(\ell+1)}=2^{-(j+1)}$.
\end{lemma}

\begin{proof}
Landing at integer $i\in\{1,\ldots,j+1\}$ means $G_{t+1}=j+1-i$, so
$\Pr(W_{t+1}=i\mid W_t=j)=\Pr(G_{t+1}=j+1-i)=2^{-(j+2-i)}$.
Landing at $0$ or below means $G_{t+1}\ge j+1$, so
$\Pr(W_{t+1}\le0\mid W_t=j)=\sum_{\ell\ge j+1}2^{-(\ell+1)}=2^{-(j+1)}$.
\end{proof}

The transition probabilities from $W_t$ to $W_{t+1}$ in \Cref{lem:kearns-killed-walk-transition} are very closely related to the update rule for $\mu^{(p-1)}$ to $\mu^{(p)}$ in \Cref{lem:kearns-tail-shape}. To make the transition explicit, we next consider the killed walk. We call
the mass that lands at $0$ or below \emph{killed}: once its integer reaches $0$
or below, it is removed.

Let $\tau$ be the time the mass is killed:
\begin{equation*}
  \tau=
  \begin{cases}
    \min\{t\ge0:W_t\le0\}, & \text{if this set is nonempty},\\
    \infty, & \text{otherwise}.
  \end{cases}
\end{equation*}
For each integer $t\ge0$, define $\nu_t$ by looking at time $t$ only among the
outcomes that have not been killed: $\nu_t(i)=\Pr(W_t=i\mid \tau>t)$ for
$i\ge1$. Thus $\nu_0$ puts all mass at integer $1$.

We now show the concrete connection between the killed walk and the tail probabilities $\mu^{(p)}$.

\begin{restatable}[Tail as a killed-walk mixture]{lemma}{lemKearnsTailMixture}
\label{lem:kearns-tail-mixture}
For every $2\le p\le k-1$, the vector $\mu^{(p)}$ from
\Cref{lem:kearns-tail-shape} is a convex combination of
$\nu_0,\ldots,\nu_{p-2}$.
\end{restatable}

The proof is in \Cref{sec:omitted-proofs}. Thus, to bound the second moment of $\mu^{(p)}$, we need to bound the second moment of the distributions $\nu_0, \ldots, \nu_{p-2}$. We state this as the following lemma, which is the main target of the rest of \Cref{sec:kearns-cyclic-example}.

\begin{lemma}[Killed-walk moment]
\label{lem:kearns-killed-walk-moment}
For every integer $t\ge0$, the distribution $\nu_t$ satisfies
\begin{equation*}
\sum_i i^2\nu_t(i)=\E[W_t^2\mid\tau>t]\le9(t+1).
\end{equation*}
\end{lemma}

Having the above lemma, we can directly show \Cref{lem:kearns-mu-second-moment}.

\begin{proof}[Proof of \Cref{lem:kearns-mu-second-moment}]
\Cref{lem:kearns-tail-mixture} writes $\mu^{(p)}$ as a convex
combination of $\nu_0,\ldots,\nu_{p-2}$. The killed-walk moment bound in \Cref{lem:kearns-killed-walk-moment} then gives
\begin{equation*}
  \sum_i i^2\mu_i^{(p)}
  \le
  \max_{0\le t\le p-2} \sum_i i^2 \nu_t(i)
  \le
  9(p-1)
  \le
  9p. \qedhere
\end{equation*}
\end{proof}

\subsection{Killed-walk Second Moment Bound}
\label{subsec:catalan}

To tackle \Cref{lem:kearns-killed-walk-moment}, we use the decomposition
\begin{equation}
  \label{eq:kearns-killed-walk-moment-decomp}
  \E[W_t^2 | \tau > t] = \frac{\E[W_t^2 \mathbf{1}_{\{\tau > t\}}]}{\Pr(\tau > t)}.
\end{equation}
We bound the numerator and denominator separately.

First, we compute $\Pr(\tau > t)$. We will construct the generating function $H(z) = \sum_{t\ge0} \Pr(\tau > t) z^t$ and give a closed form for it. Intuitively, this distribution is closely related to the Catalan numbers. We will use two tools. The Catalan generating function is classical, so we state it without proof; the generalization we need is less standard, so we derive it with a short lemma.

Let $C_n$ be the $n$-th Catalan number. The Catalan generating function is $C(s) = \sum_{i\ge0} C_i s^i$. This generating function satisfies the functional equation $C(s) = 1 + sC(s)^2$ \citep{stanley2015catalan}, and hence
\begin{equation}
  \label{eq:catalan-gf}
  C(s) = \frac{1-\sqrt{1-4s}}{2s}.
\end{equation}

We will consider words over the alphabet $\{ (, ) \}$, i.e. strings of parentheses. We call a word \emph{valid} if every prefix of the word has at least as many opening parentheses as closing parentheses. We call a word \emph{balanced} if it has the same number of opening and closing parentheses. For a word $w$, we define $o(w)$ to be the number of opening parentheses in $w$ and $c(w)$ to be the number of closing parentheses in $w$.

Define the bivariate generating function $F(a,b)$ on valid balanced words as
\begin{equation*}
  F(a,b)=\sum_{\text{valid balanced } w} a^{o(w)}b^{c(w)}.
\end{equation*}
Since in balanced words the number of opening and closing parentheses is the same, we have $F(a,b)=C(ab)$.
Next, define the bivariate generating function $G(a,b)$ on valid words as
\begin{equation*}
  G(a,b)=\sum_{\text{valid } w} a^{o(w)}b^{c(w)}.
\end{equation*}
The next lemma will establish a closed-form expression for $G(a,b)$.

\begin{restatable}{lemma}{lemKearnsCatalanPrefixGF}
\label{lem:kearns-catalan-prefix-gf}
\begin{equation*}
  G(a,b)
  =
  \frac{C(ab)}{1-aC(ab)}.
\end{equation*}
\end{restatable}

\begin{restatable}[Survival words]{lemma}{lemKearnsSurvivalWords}
\label{lem:kearns-survival-words}
At time $t$, form a word $w$ by writing, for each step $s=1,\ldots,t$, one opening parenthesis followed by $G_s$ closing parentheses. Then $w$ is valid if and only if the mass has not been killed by time $t$.
\end{restatable}

The proofs of the above two lemmas are in \Cref{sec:omitted-proofs}. We next use the fact that $H(z) = G(z/2,1/2)$ to get a closed form for $H$. This will later be used to get a closed form for $\Pr(\tau > t)$.

\begin{lemma}[Survival generating function]
\label{lem:kearns-survival-gf}
\begin{equation*}
  H(z) = \sum_{t\ge0}\Pr(\tau>t)z^t
  =
  \frac2z\bigl((1-z)^{-1/2}-1\bigr).
\end{equation*}
\end{lemma}

\begin{proof}
Fix $t$. For a sequence of outcomes $G_1, \ldots, G_t$, construct a word $w$ as
in \Cref{lem:kearns-survival-words}. Then $w$ has $o(w)=t$ and
$c(w)=\sum_{s=1}^t G_s$. By \Cref{lem:kearns-survival-words}, $w$ is valid if
and only if $\tau > t$.

The probability of this outcome is $\prod_{s=1}^t2^{-(G_s+1)}=2^{-(o(w)+c(w))}$. Then write
\begin{align*}
  H(z) = \sum_{t\ge0}\Pr(\tau>t)z^t = \sum_{\text{valid } w} 2^{-(o(w)+c(w))}z^{o(w)} = \sum_{\text{valid } w} (z/2)^{o(w)} 2^{-c(w)}
  = G(z/2, 1/2).
\end{align*}

By \Cref{lem:kearns-catalan-prefix-gf}, we have
\begin{equation*}
  H(z) = \frac{C(z/4)}{1-(z/2)C(z/4)}.
\end{equation*}
Applying \Cref{eq:catalan-gf} with $s=z/4$ and simplifying gives the claimed form.
\end{proof}

The closed form for $H$ is useful because its coefficients are standard central
binomial coefficients. Extracting them gives an exact survival probability.

\begin{lemma}[Coefficient extraction]
\label{lem:kearns-central-binomial-coefficient}
For every integer $t\ge0$, the coefficient of $z^t$ in
$\frac2z((1-z)^{-1/2}-1)$ is $\binom{2t+2}{t+1}/2^{2t+1}$.
\end{lemma}

\begin{proof}
By the generalized binomial theorem,
\begin{align*}
  (1-z)^{-1/2}
  &=
  \sum_{n\ge0}\binom{-1/2}{n}(-z)^n,\\
  \binom{-1/2}{n}(-1)^n
  &=
  \frac{1\cdot3\cdots(2n-1)}{2^n n!}
  =
  \frac{(2n)!}{4^n(n!)^2}
  =
  \frac1{4^n}\binom{2n}{n}.
\end{align*}
Therefore
$(1-z)^{-1/2}=\sum_{n\ge0}\binom{2n}{n}z^n/4^n$. Subtracting $1$ removes the
$n=0$ term, and dividing by $z$ shifts the power down by one. Hence the
coefficient of $z^t$ in $H(z) = \frac2z((1-z)^{-1/2}-1)$ is
$2\binom{2t+2}{t+1}/4^{t+1}$ which equals $\binom{2t+2}{t+1}/2^{2t+1}$.
\end{proof}

The above lemma gives a closed form for $\Pr(\tau>t)$. The next lemma provides bounds on this probability. Its proof is in \Cref{sec:omitted-proofs}.

\begin{restatable}[Survival probability]{lemma}{lemKearnsWalkSurvival}
\label{lem:kearns-walk-survival}
For every integer $t\ge0$,
\begin{math}
  1/{\sqrt{t+1}}
  \le
  \Pr(\tau>t)
  \le
  2/{\sqrt{t+1}}.
\end{math}
\end{restatable}

Having bounded the denominator of \Cref{eq:kearns-killed-walk-moment-decomp}, we next aim to bound the numerator $\E[W_t^2 \mathbf{1}_{\{\tau > t\}}]$ in the next three lemmas. This is the last step before proving \Cref{lem:kearns-killed-walk-moment}. The proofs of the following two lemmas are in \Cref{sec:omitted-proofs}.

\begin{restatable}[Increment moments]{lemma}{lemKearnsIncrementMoments}
\label{lem:kearns-increment-moments}
For each integer $t\ge1$, the increment $1-G_t$ satisfies
$\E[1-G_t]=0$ and $\operatorname{Var}(1-G_t)=2$.
\end{restatable}

\begin{restatable}{lemma}{lemKearnsStoppedTimeMean}
\label{lem:kearns-stopped-time-mean}
For every integer $t\ge0$, $\E[\min\{t,\tau\}]\le 4\sqrt t$.
\end{restatable}

\begin{lemma}
\label{lem:kearns-stopped-second-moment}
For every integer $t\ge0$,
$\E[W_t^2\mathbf 1_{\{\tau>t\}}]\le 1+8\sqrt t$.
\end{lemma}

\begin{proof}
Put $T=\min\{t,\tau\}$. Since $T$ is an integer between $0$ and $t$,
\begin{equation*}
  W_T^2
  =
  W_0^2+\sum_{s=0}^{t-1}\mathbf 1_{\{T>s\}}(W_{s+1}^2-W_s^2).
\end{equation*}

Fix $s<t$. The event $\{T>s\}$ is determined by $G_1,\ldots,G_s$. After fixing
these values, $W_s$ is fixed and $G_{s+1}$ is independent. Since
$W_{s+1}=W_s+1-G_{s+1}$, \Cref{lem:kearns-increment-moments} gives
\begin{equation*}
  \E[W_{s+1}^2-W_s^2\mid G_1,\ldots,G_s]
  =
  2W_s\E[1-G_{s+1}]+\E[(1-G_{s+1})^2]
  =
  2.
\end{equation*}
Thus
\begin{equation*}
  \E[\mathbf 1_{\{T>s\}}(W_{s+1}^2-W_s^2)]
  =
  2\Pr(T>s).
\end{equation*}
Taking expectations in the telescoping identity gives
\begin{equation*}
  \E[W_T^2]
  =
  1+2\sum_{s=0}^{t-1}\Pr(T>s)
  =
  1+2\E[T],
\end{equation*}
where the last equality uses again that $T$ is integer-valued and lies between
$0$ and $t$. By \Cref{lem:kearns-stopped-time-mean},
$\E[W_T^2]\le1+8\sqrt t$. On $\tau>t$, we have $T=t$ and $W_T=W_t$. On $\tau\le t$, the term $W_T^2$ is
nonnegative, so
\begin{equation*}
  \E[W_T^2]
  =
  \E[W_t^2\mathbf 1_{\{\tau>t\}}]
  +
  \E[W_T^2\mathbf 1_{\{\tau\le t\}}]
  \ge
  \E[W_t^2\mathbf 1_{\{\tau>t\}}].
\end{equation*}
Combining the two inequalities proves the lemma.
\end{proof}

We finally have all the necessary tools to prove \Cref{lem:kearns-killed-walk-moment}.

\begin{proof}[Proof of \Cref{lem:kearns-killed-walk-moment}]
The equality follows from the definition
$\nu_t(i)=\Pr(W_t=i\mid\tau>t)$. For the upper bound, combine
\Cref{lem:kearns-walk-survival,lem:kearns-stopped-second-moment}:
\begin{equation*}
\E[W_t^2\mid\tau>t]=\frac{\E[W_t^2\mathbf 1_{\{\tau>t\}}]}{\Pr(\tau>t)}
\le(1+8\sqrt t)\sqrt{t+1}\le9(t+1).
\end{equation*}
The last line uses $1\le\sqrt{t+1}$ and $\sqrt t\le\sqrt{t+1}$.
\end{proof}

\section{Depth-Dependent Lower Bound and Fixed-Distribution Obstruction}
\label{sec:depth-dependent-obstruction}

The cyclic example of the previous section only reaches depth $D<M^2$, and even there its bound is weak since it shrinks with a polynomial rate in $D$. We now close the picture. For each $M$ and $D$ with $D\ge M^2$, we build an $M$-covered path whose excess error is $\Omega(M^2/D)$. As a special case, setting $D=M^2$ shows the error can stay bounded away from zero all the way up to the quadratic scale.

An artifact of this construction is that the distribution itself depends on the specific value of $D$. One might aim to show the lower bound with a distribution only depending on $M$, similar to the cyclic example in \Cref{sec:kearns-cyclic-example}. We show that this is not possible: under any fixed distribution, the path predictor converges to the global predictor at a geometric rate along every $M$-covered path, so no fixed distribution can witness the $M^2/D$ bound at every depth.

\subsection{The Depth-Dependent Lower Bound}

In this section, we will introduce the instance with $\Omega(M^2/D)$ excess error. The plan is to make information aggregation as slow as possible. We hide a target signal of small size $\rho$ in a direction that no single feature reveals, and we let the feature seen by successive agents rotate by only a tiny angle $\delta$ at each step. Because consecutive features point in almost the same direction, while the residual is already orthogonal to the current feature, each agent can remove only a tiny fraction of the remaining error. Matching $\rho$ and $\delta$ to the depth $D$ keeps the error at $\Omega(M^2/D)$ even after $D$ agents.

Fix integers $M\ge8$ and $D\ge M^2$. Define
\begin{equation*}
  \delta=\frac{2\pi}{M},
  \qquad
  \rho^2=\frac{1}{80D\sin^2\delta}.
\end{equation*}
Let $X_\star,Z_0,Z_1,Z_2$ be independent standard Gaussians, and let
$Y=X_\star+\rho Z_0$. We will directly give $X_\star$ to all agents. We also define features $X_0, X_1, \ldots, X_{M-1}$ as follows. For $j=0,\ldots,M-1$,
\begin{equation*}
  X_j=Z_1+\rho(\cos(j\delta)Z_0+\sin(j\delta)Z_2).
\end{equation*}
Agent $A_t$ sees $X_\star$, $X_{(t-1)\bmod M}$, and the prediction from
$A_{t-1}$ if $t > 1$. Thus every block of $M$ agents sees all features.

The common feature $X_\star$ is there only to keep the scale of the labels near 1. Every agent sees $X_\star$, so this part of the label is learned immediately and
is orthogonal to the variables $Z_0,Z_1,Z_2$ that drive the lower bound. All the
hard dynamics therefore live in the three-dimensional subspace with target
$\rho Z_0$. Dropping $X_\star$ and using the label $\rho Z_0$ would have a similar analysis, but then
the whole label would have variance only $\rho^2$. Keeping $X_\star$ makes
$\operatorname{Var}(Y)=1+\rho^2 \le 2$.

In this section, we will prove that this construction yields the following result.

\longDepthRestatement*

The common feature $X_\star$ is seen by every agent
and is independent of everything else, so it is learned at once and the whole
difficulty lives in the three-dimensional span of $Z_0,Z_1,Z_2$. For the rest of this section, we call the remaining part the \emph{hard part}. We call the features $X_j$ the \emph{hard features} and $Y - X_\star$ the \emph{hard label}.

We track the
error $E_t$ of the path in this hard part and prove two facts about it: it starts
at a constant fraction of $\rho^2$ (\Cref{lem:first-two-hard-features}), and each
later agent removes at most an $O(\rho^2\sin^2\delta)$ fraction of it, that is
$E_{t+1}\ge(1-O(\rho^2\sin^2\delta))E_t$ (\Cref{lem:slow-drift}). With the
choice $\rho^2=\Theta(1/(D\sin^2\delta))$ the
per-step factor is $1-\Theta(1/D)$, so a constant fraction of the error survives
all $D$ steps, leaving $E_D=\Omega(\rho^2)=\Omega(M^2/D)$.

We now set up the hard part. After subtracting the common feature, everything
that remains lives in the span of $Z_0,Z_1,Z_2$. These are independent standard
Gaussians, hence orthonormal, so every hard predictor and residual is determined
by its coefficient vector in $\mathbb{R}^3$. Let $e_0,e_1,e_2$
be the coefficient vectors of $Z_0,Z_1,Z_2$, that is the standard basis of
$\mathbb{R}^3$, and set
\begin{equation*}
  u(\theta)=\cos\theta\,e_0+\sin\theta\,e_2,
  \qquad
  x(\theta)=e_1+\rho u(\theta).
\end{equation*}
The hard label $Y-X_\star=\rho Z_0$ has coefficient vector $\rho e_0$, and the
hard feature $X_j$ has coefficient vector $x(j\delta)$. Agent $A_t$ sees
$X_{(t-1)\bmod M}$, so the feature direction
$u(\theta)$ rotates by $\delta$ from one agent to the next. Let $\widehat{y}^{\mathrm{h}}_t\in\mathbb{R}^3$ be the coefficient vector of the
hard prediction after $t$ agents, and define the residual, hard error, and $s_t$ by
\begin{equation}\label{eq:hard-defs}
  R_t:=\rho e_0-\widehat{y}^{\mathrm{h}}_t,
  \qquad
  E_t:=\norm{R_t}_2^2,
  \qquad
  s_t:=1-\frac{E_t}{\rho^2}.
\end{equation}
In contrast to \Cref{sec:kearns-cyclic-example} where $R_t$ was a scalar, here $R_t$ is a vector in $\mathbb{R}^3$.
The hard prediction at agent $A_t$ is the least-squares projection of $\rho e_0$
onto the span of $\widehat{y}^{\mathrm{h}}_{t-1}$ and $x((t-1)\delta)$ in
$\mathbb{R}^3$. By \Cref{lem:prelim-ls-projection}, $R_t$ is orthogonal to the
fitted span, hence to $\widehat{y}^{\mathrm{h}}_t$. Pythagoras on
$\rho e_0=\widehat{y}^{\mathrm{h}}_t+R_t$ then gives
$\norm{\widehat{y}^{\mathrm{h}}_t}_2^2+E_t=\rho^2$, equivalently
\begin{equation}\label{eq:hard-norms}
  \norm{\widehat{y}^{\mathrm{h}}_t}_2=\rho\sqrt{s_t},
  \qquad
  \norm{R_t}_2=\rho\sqrt{1-s_t}.
\end{equation}
In particular $s_t\in[0,1]$.

We analyze the starting error at the second agent. This agent sees $x(\delta)$, and the prediction of the first agent, which sees $x(0)$. We will show that the prediction of the first agent is a nonzero multiple of $x(0)$, so the second agent sees $x(0)$ too. We analyze the error of the best predictor to these two features in the following lemma.

\begin{lemma}
\label{lem:first-two-hard-features}
After the first two hard features,
\begin{equation*}
  E_2=
  \frac{\rho^2}{1+\rho^2+\tan^2(\delta/2)}.
\end{equation*}
Also,
\begin{equation*}
  s_2=
  \frac{\rho^2+\tan^2(\delta/2)}
       {1+\rho^2+\tan^2(\delta/2)}.
\end{equation*}
\end{lemma}

\begin{proof}
Agent $A_1$ sees $x(0)$, and its prediction is a nonzero multiple of $x(0)$, because otherwise $R_1 = \rho e_0$ and $R_1$ must be orthogonal to $x(0)$, but is not. Agent $A_2$ sees that prediction together with $x(\delta)$, so by \Cref{lem:prelim-ls-projection} it projects $\rho e_0$
onto the plane spanned by $x(0)$ and $x(\delta)$. Thus $E_2$ is the squared
distance from $\rho e_0$ to that plane.

The two features have length $\sqrt{1 + \rho^2}$, so the sum $g=x(0)+x(\delta)$ and difference
$d=x(\delta)-x(0)$ are orthogonal, $\langle g,d\rangle=\norm{x(\delta)}_2^2-\norm{x(0)}_2^2=0$,
and give an orthogonal basis of the plane. Using $\langle e_0,u(\theta)\rangle=\cos\theta$,
\begin{align*}
  \langle\rho e_0,g\rangle&=\rho^2(1+\cos\delta), & \norm{g}_2^2&=4+2\rho^2(1+\cos\delta),\\
  \langle\rho e_0,d\rangle&=-\rho^2(1-\cos\delta), & \norm{d}_2^2&=2\rho^2(1-\cos\delta).
\end{align*}
The distance is $\rho e_0$ minus its projections onto $g$ and $d$. Substituting the
four quantities above, simplifying each fraction, and combining over a common
denominator,
\begin{align*}
  E_2
  &=\norm{\rho e_0}_2^2-\frac{\langle\rho e_0,g\rangle^2}{\norm{g}_2^2}
     -\frac{\langle\rho e_0,d\rangle^2}{\norm{d}_2^2}\\
  &=\frac{\rho^2(1+\cos\delta)\,[\,2+\rho^2(1+\cos\delta)\,]-\rho^4(1+\cos\delta)^2}
         {2\,[\,2+\rho^2(1+\cos\delta)\,]}\\
  &=\frac{\rho^2(1+\cos\delta)}{2+\rho^2(1+\cos\delta)}\\
  \intertext{Rewrite using $1+\cos\delta=2\cos^2(\delta/2)$ then divide top and bottom by $\cos^2(\delta/2)$ and use $1/\cos^2(\delta/2)=1+\tan^2(\delta/2)$,}
  &=\frac{\rho^2\cos^2(\delta/2)}{1+\rho^2\cos^2(\delta/2)}
  =\frac{\rho^2}{1+\rho^2+\tan^2(\delta/2)}.
\end{align*}
Finally, using $s_2=1-E_2/\rho^2$ from the setup,
\begin{equation*}
  s_2=1-\frac{E_2}{\rho^2}
     =1-\frac{1}{1+\rho^2+\tan^2(\delta/2)}
     =\frac{\rho^2+\tan^2(\delta/2)}{1+\rho^2+\tan^2(\delta/2)}. \qedhere
\end{equation*}
\end{proof}

The heart of the construction is that, from the third agent on, one agent barely helps. The
residual $R_t$ is already orthogonal to the feature $x(\theta)$ that agent $A_t$ just
saw. The next feature $x(\theta+\delta)$ differs from $x(\theta)$ only by a vector
of length $2\rho\sin(\delta/2)$, so it too is nearly orthogonal to $R_t$, and a
least-squares step along a direction nearly orthogonal to the residual removes
almost none of it.

\begin{lemma}[One slow step]
\label{lem:slow-drift}
Suppose $M\ge8$ and $0<\rho\le1/8$. For every $t\ge2$,
\begin{equation*}
  E_{t+1}\ge(1-40\rho^2\sin^2\delta)E_t.
\end{equation*}
\end{lemma}

Having the above lemma, we can prove \Cref{thm:long-depth}.

\begin{proof}[Proof of \Cref{thm:long-depth}]
First check the global predictor. Around the full circle,
\begin{equation*}
  \sum_{j=0}^{M-1}\cos(j\delta)=0,
  \qquad
  \sum_{j=0}^{M-1}\cos(j\delta)\sin(j\delta)=0,
  \qquad
  \sum_{j=0}^{M-1}\cos^2(j\delta)=\frac{M}{2}.
\end{equation*}
Hence
\begin{equation*}
  Y=X_\star+\frac{2}{M}\sum_{j=0}^{M-1}\cos(j\delta)X_j.
\end{equation*}
This predictor is exact, and its coefficient $\ell_1$ norm is at most
$1+(2/M)\sum_j|\cos(j\delta)|\le3$.

We also need $\rho$ to be small enough for \Cref{lem:slow-drift}. Since
$\delta\le\pi/4$, the bound $\sin x\ge2x/\pi$ on $[0,\pi/2]$ gives
$\sin\delta\ge4/M$. Since $D\ge M^2$,
\begin{equation*}
  D\sin^2\delta\ge16,
  \qquad
  \rho^2\le\frac{1}{1280}.
\end{equation*}
Thus $\rho\le1/8$, and the feature second moments are
$\E[X_\star^2]=1$ and $\E[X_j^2]=1+\rho^2\le2$.

The common feature $X_\star$ is independent of the hard variables and is
available at every agent. Therefore the full prediction is $X_\star$ plus the hard
prediction. \Cref{lem:slow-drift} and the choice of
$\rho$ give
\begin{equation*}
  E_D
  \ge
  E_2(1-40\rho^2\sin^2\delta)^{D-2}
  =
  E_2\left(1-\frac{1}{2D}\right)^{D-2}.
\end{equation*}
By Bernoulli's inequality,
$(1-1/(2D))^{D-2}\ge1-(D-2)/(2D)\ge1/2$. Also
$\tan(\delta/2)=\tan(\pi/M)\le\tan(\pi/8)<1/2$ and
$\rho^2\le1/1280$, so \Cref{lem:first-two-hard-features} gives
$E_2\ge\rho^2/2$. Hence
\begin{equation*}
  E_D
  \ge
  \frac{\rho^2}{4}
  =
  \frac{1}{320D\sin^2(2\pi/M)}
  \ge
  \frac{M^2}{1280\pi^2D}.
\end{equation*}
Since the global predictor is exact, $\MSE(f^\star)=0$ and $X_\star$ is learned with no
error, so this hard-part error $E_D$ equals the excess error
$\MSE(f_D)-\MSE(f^\star)$ at agent $A_D$.
\end{proof}

Setting $D=M^2$ directly shows \Cref{cor:constant-before-quadratic}.

\constantBeforeQuadraticRestatement*

\begin{proof}
Apply \Cref{thm:long-depth} with target depth $D=M^2$. By \Cref{lem:prelim-mse-decomposition}, $E_t$ is non-increasing along the path. Thus, for every $1\le t\le M^2$,
$E_t\ge E_{M^2}\ge1/(1280\pi^2)$.
\end{proof}

\subsection[Proof of Lemma \ref*{lem:slow-drift}]{Proof of \Cref{lem:slow-drift}}

The residual $R_t$ is orthogonal to the feature $x(\theta)=x((t-1)\delta)$
that agent $A_t$ just saw, and the next feature $x(\theta+\delta)$ differs from
$x(\theta)$ by a vector of length only $2\rho\sin(\delta/2)$. So
$x(\theta+\delta)$ is also nearly orthogonal to $R_t$, and the least-squares
step on it removes only an $O(\rho^2\sin^2(\delta/2))$ fraction of $E_t$.

Three lemmas formalize this. \Cref{lem:slow-drift-setup} sets up an orthonormal
basis based on the prediction and residual of agent $A_t$, and \Cref{lem:slow-drift-B-bounds,lem:slow-drift-numerator} supply some bounds in that basis. We combine these lemmas in the proof of \Cref{lem:slow-drift}.

\begin{lemma}[Setup]
\label{lem:slow-drift-setup}
Fix $t\ge2$ and assume $E_t>0$. Set $\theta=(t-1)\delta$ and define
\begin{equation*}
  m:=\frac{\widehat{y}^{\mathrm{h}}_t}{\norm{\widehat{y}^{\mathrm{h}}_t}_2},
  \qquad
  n:=\frac{R_t}{\norm{R_t}_2};
\end{equation*}
write $n=(n_0,n_1,n_2)$ in the basis $e_0,e_1,e_2$, and set
\begin{equation*}
  h:=\frac{(0,n_2,-n_1)}{\sqrt{s_t}}.
\end{equation*}
Then $\{m,n,h\}$ is an orthonormal basis of $\mathbb{R}^3$, and there exist
unique real numbers $a,b$ with $x(\theta)=a\,m+b\,h$. Define $P(z)=z-\langle z,m\rangle m$. Setting
$B=|b|$,
\begin{equation*}
  B=\norm{P(x(\theta))}_2,
  \qquad
  B^2=1+\rho^2-\frac{\rho^2\cos^2\theta}{s_t}.
\end{equation*}
\end{lemma}

\begin{proof}
\Cref{lem:prelim-ls-projection} applied to agent $A_t$ gives that $R_t$ is
orthogonal to both $\widehat{y}^{\mathrm{h}}_t$ and $x(\theta)$. The error is non-increasing
along the path (\Cref{lem:prelim-mse-decomposition}), so $E_t\le E_2$, hence
$s_t\ge s_2$; and $s_2>0$ by \Cref{lem:first-two-hard-features}. Combined with
the assumption $E_t>0$, this gives $0<s_t<1$, so the denominators in the
definitions of $m,n,h$ are nonzero.

Using \Cref{eq:hard-norms}, the identity
$\rho e_0=\widehat{y}^{\mathrm{h}}_t+R_t$ rewrites as
\begin{equation*}
e_0=\sqrt{s_t}\,m+\sqrt{1-s_t}\,n.
\end{equation*}
The vectors $m$ and $n$ are perpendicular
unit vectors since $R_t$ is orthogonal to $\widehat{y}^{\mathrm{h}}_t$. Taking
inner product of the previous identity with $n$ gives
$n_0=\langle e_0,n\rangle=\sqrt{1-s_t}$, so $n_1^2+n_2^2=1-n_0^2=s_t$.

For $h$: $\norm{h}_2=\sqrt{n_1^2+n_2^2}/\sqrt{s_t}=1$; $\langle h,n\rangle=0$ by
direct check; $\langle h,e_0\rangle=0$ since $h$ has zero first coordinate; and
$\langle h,m\rangle=0$ since $m\in\mathrm{span}\{e_0,n\}$ by the previous
identity. Hence $\{m,n,h\}$ is orthonormal.

Since $R_t$ is orthogonal to $x(\theta)$, the vector $x(\theta)$ has no
$n$-component, so $x(\theta)=a\,m+b\,h$ for unique real $a,b$, and
$P(x(\theta))=b\,h$, giving $B=|b|=\norm{P(x(\theta))}_2$. From
$\norm{x(\theta)}_2^2=1+\rho^2$ and orthonormality of $\{m,h\}$,
$a^2+B^2=1+\rho^2$. Computing
$\langle\widehat{y}^{\mathrm{h}}_t,x(\theta)\rangle$ in two ways, using
$\widehat{y}^{\mathrm{h}}_t=\rho\sqrt{s_t}\,m$ on the left, and on the right that
$R_t$ is orthogonal to $x(\theta)$ together with
$\langle e_0,x(\theta)\rangle=\rho\cos\theta$,
\begin{equation*}
  \rho\sqrt{s_t}\,a
  =\langle\widehat{y}^{\mathrm{h}}_t,x(\theta)\rangle
  =\langle\rho e_0,x(\theta)\rangle
  =\rho^2\cos\theta.
\end{equation*}
Therefore $a=\rho\cos\theta/\sqrt{s_t}$, and substituting into
$a^2+B^2=1+\rho^2$ gives $B^2=1+\rho^2-\rho^2\cos^2\theta/s_t$.
\end{proof}

Throughout the rest of this subsection, we use the notation of
\Cref{lem:slow-drift-setup}.

\begin{lemma}[Two lower bounds on $B$]
\label{lem:slow-drift-B-bounds}
$B\ge\sin(\delta/2)$ and $B\ge|\sin\theta|$.
\end{lemma}

\begin{proof}
The error is non-increasing along the path
(\Cref{lem:prelim-mse-decomposition}), so $s_t\ge s_2$, and
\Cref{lem:first-two-hard-features} gives
\begin{equation*}
  s_2=\frac{\rho^2+\tan^2(\delta/2)}{1+\rho^2+\tan^2(\delta/2)}.
\end{equation*}
For the first bound, $s_t\ge s_2$ yields
\begin{equation*}
  \frac{1}{s_t}\le\frac{1+\rho^2+\tan^2(\delta/2)}{\rho^2+\tan^2(\delta/2)}.
\end{equation*}
Plugging into the formula for $B^2$ gives
\begin{equation*}
  B^2\ge1+\rho^2-\rho^2\cos^2\theta\,\frac{1+\rho^2+\tan^2(\delta/2)}{\rho^2+\tan^2(\delta/2)},
\end{equation*}
and combining over the common denominator $\rho^2+\tan^2(\delta/2)$ (using
$\cos^2\theta=1-\sin^2\theta$),
\begin{equation*}
  B^2
  \ge\frac{\tan^2(\delta/2)+\rho^2(1+\rho^2+\tan^2(\delta/2))\sin^2\theta}{\rho^2+\tan^2(\delta/2)}
  \ge\frac{\tan^2(\delta/2)}{\rho^2+\tan^2(\delta/2)}.
\end{equation*}
For $M\ge8$ we have $\tan(\delta/2)\le\tan(\pi/8)<1/2$, and with $\rho\le1/8$ this gives
$\rho^2+\tan^2(\delta/2)<1$. Hence $B\ge\tan(\delta/2)\ge\sin(\delta/2)$.

For the second bound, rewrite the formula for $B^2$ as
\begin{equation*}
  B^2-\sin^2\theta=\rho^2+\cos^2\theta\Bigl(1-\frac{\rho^2}{s_t}\Bigr).
\end{equation*}
If $s_t\ge\rho^2$, the right side is nonnegative. Otherwise the bracket is
negative, and $\cos^2\theta\le1$ gives $B^2-\sin^2\theta\ge1+\rho^2-\rho^2/s_t$.
The formula for $s_2$ gives
\begin{equation*}
  s_2-\frac{\rho^2}{1+\rho^2}
  =\frac{\tan^2(\delta/2)}{(1+\rho^2)(1+\rho^2+\tan^2(\delta/2))}\ge0,
\end{equation*}
so $s_t\ge s_2\ge\rho^2/(1+\rho^2)$, that is $\rho^2/s_t\le1+\rho^2$, and the
last lower bound is again nonnegative. In either case $|\sin\theta|\le B$.
\end{proof}

\begin{lemma}[Numerator bound]
\label{lem:slow-drift-numerator}
$|\langle n,x(\theta+\delta)\rangle|\le8\rho\sin(\delta/2)\,B$.
\end{lemma}

\begin{proof}
We use both bounds from \Cref{lem:slow-drift-B-bounds}: $\sin(\delta/2)\le B$
and $|\sin\theta|\le B$. The orthogonality $\langle n,x(\theta)\rangle=0$
follows from $n=R_t/\norm{R_t}_2$ and the fact that $R_t$ is orthogonal to
$x(\theta)$ (\Cref{lem:prelim-ls-projection}). The plan is to use this
orthogonality together with the fact that the two features differ only by a
short vector, so we just need to control how $n$ pairs with that short vector. Expanding
$\langle n,x(\theta)\rangle=0$ with
$x(\theta)=e_1+\rho(\cos\theta\,e_0+\sin\theta\,e_2)$,
\begin{equation*}
  n_1=-\rho(n_0\cos\theta+n_2\sin\theta).
\end{equation*}
We first bound $|n_2|$. The identity $b=\langle x(\theta),h\rangle$ together
with $\langle e_1,h\rangle=n_2/\sqrt{s_t}$ and $\langle e_2,h\rangle=-n_1/\sqrt{s_t}$
gives $b=(n_2-\rho n_1\sin\theta)/\sqrt{s_t}$; substituting the formula for
$n_1$,
\begin{equation*}
  B\sqrt{s_t}
  =\bigl|n_2(1+\rho^2\sin^2\theta)+\rho^2n_0\sin\theta\cos\theta\bigr|.
\end{equation*}
By the triangle inequality, $|n_0|\le1$, $|\cos\theta|\le1$, and
$\sqrt{s_t}\le1$, the previous display gives
$|n_2|(1+\rho^2\sin^2\theta)\le B+\rho^2|\sin\theta|$. Dividing by
$1+\rho^2\sin^2\theta\ge1$ and using $|\sin\theta|\le B$ and $\rho^2\le1/64$,
\begin{equation*}
  |n_2|\le B+\rho^2|\sin\theta|\le\tfrac{65}{64}B<\tfrac32 B.
\end{equation*}

Now define $v(\phi)=-\sin\phi\,e_0+\cos\phi\,e_2$, a unit vector orthogonal to
$u(\phi)$. Standard angle-addition identities give
\begin{equation*}
  u(\phi+\alpha)-u(\phi)=2\sin(\alpha/2)\,v(\phi+\alpha/2),
  \qquad
  v(\phi+\alpha)=\cos\alpha\,v(\phi)-\sin\alpha\,u(\phi).
\end{equation*}
The first identity at $\phi=\theta$, $\alpha=\delta$ gives
\begin{equation*}
  x(\theta+\delta)-x(\theta)=2\rho\sin(\delta/2)\,v(\theta+\delta/2).
\end{equation*}
We bound $|\langle n,v(\theta+\delta/2)\rangle|$ in two steps. First, by the
triangle inequality and $|n_0|,|\cos\theta|\le1$,
\begin{equation*}
  |\langle n,v(\theta)\rangle|
  =|-n_0\sin\theta+n_2\cos\theta|
  \le|\sin\theta|+|n_2|
  \le B+\tfrac32 B<3B.
\end{equation*}
Second, the second angle-addition identity at $\phi=\theta$, $\alpha=\delta/2$
expands the inner product as
\begin{equation*}
  \langle n,v(\theta+\delta/2)\rangle
  =\cos(\delta/2)\,\langle n,v(\theta)\rangle-\sin(\delta/2)\,\langle n,u(\theta)\rangle.
\end{equation*}
Taking absolute values and using the bound $|\langle n,v(\theta)\rangle|<3B$ from
above, $|\langle n,u(\theta)\rangle|\le\norm{n}_2\norm{u(\theta)}_2=1$ by
Cauchy--Schwarz, and $\sin(\delta/2)\le B$ from \Cref{lem:slow-drift-B-bounds},
\begin{equation*}
  |\langle n,v(\theta+\delta/2)\rangle|
  \le\cos(\delta/2)\cdot3B+\sin(\delta/2)\cdot1
  \le3B+B=4B.
\end{equation*}
Combining and using $\langle n,x(\theta)\rangle=0$,
\begin{equation*}
  |\langle n,x(\theta+\delta)\rangle|
  =|\langle n,x(\theta+\delta)-x(\theta)\rangle|
  =2\rho\sin(\delta/2)\,|\langle n,v(\theta+\delta/2)\rangle|
  \le8\rho\sin(\delta/2)\,B. \qedhere
\end{equation*}
\end{proof}

We are finally ready to prove \Cref{lem:slow-drift}.

\begin{proof}[Proof of \Cref{lem:slow-drift}]
If $E_t=0$, the error is non-increasing along the path
(\Cref{lem:prelim-mse-decomposition}), so $E_{t+1}\le E_t=0$ and the conclusion holds trivially. Assume $E_t>0$, and
adopt the notation of \Cref{lem:slow-drift-setup}. We first derive a per-step
ratio for the relative error reduction, then plug in the two lemmas above and a
denominator bound.

By \Cref{lem:prelim-ls-projection}, agent $A_{t+1}$ projects
$\rho e_0$ onto
$V=\mathrm{span}\{\widehat{y}^{\mathrm{h}}_t,x(\theta+\delta)\}
=\mathrm{span}\{m,x(\theta+\delta)\}$.
This projection is the point of $V$ closest to $\rho e_0$, so $E_{t+1}$ is the
squared distance from $\rho e_0$ to $V$. Since $\rho e_0=\widehat{y}^{\mathrm{h}}_t+R_t$
with $\widehat{y}^{\mathrm{h}}_t$ along $m$ and $R_t$ orthogonal to $m$, the vector
$\widehat{y}^{\mathrm{h}}_t$ is already the projection of $\rho e_0$ onto the line
$\mathrm{span}\{m\}\subseteq V$, with leftover $R_t$.

Let $w:=P(x(\theta+\delta))/\norm{P(x(\theta+\delta))}_2$, the part of
$x(\theta+\delta)$ orthogonal to $m$, normalized. If $P(x(\theta+\delta))=0$, then
$V=\mathrm{span}\{m\}$, the projection is unchanged, and $E_{t+1}=E_t$, so the claim
holds. Otherwise $\{m,w\}$ is an orthonormal basis of $V$, so enlarging the line
$\mathrm{span}\{m\}$ to $V$ removes the component of the leftover $R_t$ along $w$:
the projection of $\rho e_0$ onto $V$ is $\widehat{y}^{\mathrm{h}}_t+\langle R_t,w\rangle\,w$,
leaving residual $R_t-\langle R_t,w\rangle\,w$. By Pythagoras,
\begin{equation*}
  E_{t+1}=E_t-\langle R_t,w\rangle^2.
\end{equation*}
Because $w$ is orthogonal to $m$, projecting $x(\theta+\delta)$ off $m$ does not
change its inner product with $R_t$, so
$\langle R_t,w\rangle=\langle R_t,x(\theta+\delta)\rangle/\norm{P(x(\theta+\delta))}_2$.
Substituting $R_t=\sqrt{E_t}\,n$ and dividing by $E_t$,
\begin{equation*}
  \frac{E_t-E_{t+1}}{E_t}
  =\frac{\langle n,x(\theta+\delta)\rangle^2}{\norm{P(x(\theta+\delta))}_2^2}.
\end{equation*}

\Cref{lem:slow-drift-numerator} bounds the numerator. For the denominator,
linearity of $P$ gives $P(x(\theta+\delta))=P(x(\theta))+P(x(\theta+\delta)-x(\theta))$,
so the reverse triangle inequality and the fact that $P$ does not increase
length (since $m$ is a unit vector, $\norm{z}_2^2=\langle z,m\rangle^2+\norm{P(z)}_2^2$
by Pythagoras, so $\norm{P(z)}_2\le\norm{z}_2$) give
\begin{equation*}
  \norm{P(x(\theta+\delta))}_2
  \ge\norm{P(x(\theta))}_2-\norm{x(\theta+\delta)-x(\theta)}_2
  =B-2\rho\sin(\delta/2)
  \ge B-2\rho B
  \ge\tfrac34 B,
\end{equation*}
using \Cref{lem:slow-drift-B-bounds} ($\sin(\delta/2)\le B$) and $\rho\le1/8$.
Substituting both bounds,
\begin{equation*}
  \frac{E_t-E_{t+1}}{E_t}
  \le\frac{(8\rho\sin(\delta/2)B)^2}{(\tfrac34 B)^2}
  =\frac{1024}{9}\rho^2\sin^2(\delta/2).
\end{equation*}
For $M\ge8$, $\cos^2(\delta/2)\ge\cos^2(\pi/8)>3/4$, so
$\sin^2(\delta/2)=\sin^2\delta/(4\cos^2(\delta/2))\le\sin^2\delta/3$. The
relative error reduction is therefore at most
$(1024/27)\rho^2\sin^2\delta<40\rho^2\sin^2\delta$, and
$E_{t+1}\ge(1-40\rho^2\sin^2\delta)E_t$.
\end{proof}

\subsection{The Fixed-Distribution Obstruction}

The radius $\rho$ in the construction above shrinks with $D$, which means that the distribution we introduce depends on $D$. This is not an artifact of the proof and 
we will show that the hard distribution must move with the depth.

\noFixedAllHorizonRestatement*

\begin{proof}
This follows from a geometric convergence bound for fixed distributions,
\Cref{thm:fixed-distribution-geometric} below, which gives some $q\in[0,1)$ with
$\cE_D\le\cE_1q^{\lfloor (D-1)/M\rfloor}$. The base error is bounded by a constant
of $\cD$ alone, since the zero predictor is feasible for $A_1$ and so
$\cE_1\le\MSE(f_1)\le\E[Y^2]$. Fix $c>0$. Because $q<1$, we have
$\E[Y^2]\,Dq^{\lfloor (D-1)/M\rfloor}\to0$ as $D\to\infty$, so there is a $D_c$,
depending only on $\cD$, $M$, and $c$, with $\E[Y^2]\,Dq^{\lfloor (D-1)/M\rfloor}<cM^2$
for all $D\ge D_c$. Then $\cE_D\le\E[Y^2]q^{\lfloor (D-1)/M\rfloor}<cM^2/D$.
\end{proof}

We state the following two auxiliary lemmas. Both are slightly more general than needed in this section. This is because we will also use them later for the analogous results for the binary classification model.

\begin{lemma}
\label{lem:coverage-lower-bound}
Let $H$ be a finite-dimensional space of random variables with the inner
product $\langle U,W\rangle=\E[UW]$, and let $V_1,\ldots,V_M\subseteq H$ be
subspaces with $V_1+\cdots+V_M=H$. Write $P_{V_i}$ for the orthogonal
projection onto $V_i$. Then there is a constant $\lambda>0$, depending only on
$V_1,\ldots,V_M$, such that
\begin{equation*}
  \sum_{i=1}^M\|P_{V_i}(u)\|_2^2\ge\lambda\|u\|_2^2
  \qquad\text{for every }u\in H.
\end{equation*}
\end{lemma}

\begin{proof}
The map $u\mapsto\sum_{i=1}^M\|P_{V_i}(u)\|_2^2$ is continuous on $H$ and
strictly positive on the unit sphere of $H$: if it vanished at a unit vector
$u$, then $u$ would be orthogonal to every $V_i$, hence to $V_1+\cdots+V_M=H$
by coverage, forcing $u=0$. Since $H$ is finite-dimensional, its unit sphere is
compact, so the map attains a minimum $\lambda>0$ on it. Scaling by $\|u\|_2$
gives the stated bound for every $u\in H$.
\end{proof}

The second lemma is the geometric core of the block contraction. Along a block
of agents, agent $i$ fits the features spanning $V_i$, so its residual ends up
orthogonal to $V_i$; and when the loss barely drops over the block, consecutive
residuals barely move. The lemma says that $M$-coverage then forces the part of the
starting residual that lies in $H$ to be small.

\begin{lemma}[Coverage bounds the starting residual]
\label{lem:block-residual-coverage}
Let $H$, $V_1,\ldots,V_M$, and $\lambda$ be as in
\Cref{lem:coverage-lower-bound}, and write $P_H$ for the orthogonal projection
onto $H$. Let $r_0,r_1,\ldots,r_M$ be random variables with finite second
moments such that $P_{V_i}(r_i)=0$ for every $1\le i\le M$. Then
\begin{equation*}
  \lambda\|P_H(r_0)\|_2^2
  \le
  2(M^2+1)\sum_{i=1}^M\|r_i-r_{i-1}\|_2^2 .
\end{equation*}
\end{lemma}

\begin{proof}
Write $\Delta=\sum_{i=1}^M\|r_i-r_{i-1}\|_2^2$. The plan is to bound the
projections $P_{V_i}(r_{i-1})$ from above by the steps, to bound them from
below through $P_{V_i}(r_0)$, and to apply the coverage bound of
\Cref{lem:coverage-lower-bound} to $P_H(r_0)$.

For the upper bound, $P_{V_i}(r_i)=0$ gives
$P_{V_i}(r_{i-1})=P_{V_i}(r_{i-1}-r_i)$, and orthogonal projections do not
increase norm, so
\begin{equation*}
  \sum_{i=1}^M\|P_{V_i}(r_{i-1})\|_2^2
  \le\sum_{i=1}^M\|r_{i-1}-r_i\|_2^2
  =\Delta.
\end{equation*}

For the lower bound, split
$P_{V_i}(r_{i-1})=P_{V_i}(r_0)+P_{V_i}(r_{i-1}-r_0)$, and write $a=P_{V_i}(r_0)$
and $b=P_{V_i}(r_{i-1}-r_0)$ for the two terms. Expanding the square,
\begin{equation*}
  \|P_{V_i}(r_{i-1})\|_2^2=\|a+b\|_2^2=\|a\|_2^2+2\langle a,b\rangle+\|b\|_2^2.
\end{equation*}
The cross term $2\langle a,b\rangle$ has no definite sign, so we bound it from
below. Young's inequality, in the form
$2|\langle a,b\rangle|\le\tfrac12\|a\|_2^2+2\|b\|_2^2$, gives
$2\langle a,b\rangle\ge-\tfrac12\|a\|_2^2-2\|b\|_2^2$. Substituting and
collecting the $\|b\|_2^2$ terms,
\begin{equation*}
  \|P_{V_i}(r_{i-1})\|_2^2
  \ge\|a\|_2^2-\tfrac12\|a\|_2^2-2\|b\|_2^2+\|b\|_2^2
  =\tfrac12\|a\|_2^2-\|b\|_2^2.
\end{equation*}
Since $P_{V_i}$ does not increase norm, $\|b\|_2\le\|r_{i-1}-r_0\|_2$. The
drift $r_{i-1}-r_0$ telescopes over the steps, so the triangle inequality and
Cauchy--Schwarz across $i-1\le M$ terms give
\begin{equation*}
  \|r_{i-1}-r_0\|_2^2
  =\Big\|\sum_{j=1}^{i-1}(r_j-r_{j-1})\Big\|_2^2
  \le(i-1)\sum_{j=1}^{i-1}\|r_j-r_{j-1}\|_2^2
  \le M\Delta.
\end{equation*}
Putting $a$ and $b$ back and summing over $i$,
\begin{equation*}
  \sum_{i=1}^M\|P_{V_i}(r_{i-1})\|_2^2
  \ge\tfrac12\sum_{i=1}^M\|P_{V_i}(r_0)\|_2^2-M^2\Delta.
\end{equation*}
Each $V_i\subseteq H$, so projecting onto $V_i$ factors through $H$:
$P_{V_i}(r_0)=P_{V_i}(P_H(r_0))$. Applying \Cref{lem:coverage-lower-bound} to
$P_H(r_0)\in H$ then gives
$\sum_{i=1}^M\|P_{V_i}(r_0)\|_2^2\ge\lambda\|P_H(r_0)\|_2^2$. Combining the
two bounds on $\sum_{i=1}^M\|P_{V_i}(r_{i-1})\|_2^2$,
\begin{equation*}
  \tfrac{\lambda}{2}\|P_H(r_0)\|_2^2-M^2\Delta\le\Delta,
\end{equation*}
which rearranges to the claim.
\end{proof}

To show \Cref{thm:fixed-distribution-geometric}, we will first state the following intermediate lemma. Consider a fixed distribution and a path $A_1\to\cdots\to A_M$ of agents on that distribution, in some network that may contain other agents too. Suppose agent $A_i$ sees the raw features $S_i \subseteq [d]$ and assume that $\bigcup_{i=1}^M S_i=[d]$. We will show that there exists a constant $q\in[0,1)$ such that the excess error of the last agent is at most $q$ times the excess error of the first agent. We will further show that this $q$ does not depend on the rest of the network and only depends on $S_1, \ldots, S_M$. We formalize this in the following lemma.

\begin{lemma}[Per-block contraction]
\label{lem:fixed-block-contraction}
Fix a distribution $\cD$ on $(x_1,\ldots,x_d,Y)$. Assume
$x_1,\ldots,x_d,Y$ have bounded second moments. Consider a
path $A_1\to\cdots\to A_M$ of agents on $\cD$ in any DAG, with feature
subsets $S_1,\ldots,S_M\subseteq[d]$ satisfying $\bigcup_{i=1}^M S_i=[d]$. Suppose $A_0$ is any
agent in the DAG with an edge $A_0\to A_1$. Let $\cE_t$ be the excess error of agent $A_t$ for $t=0,1,\ldots,M$. There is a constant
$q\in[0,1)$, depending only on $\cD$ and $(S_1,\ldots,S_M)$, such that
\begin{equation*}
  \cE_M\le q\,\cE_0.
\end{equation*}
\end{lemma}

\begin{proof}
Let $H=\operatorname{span}\{x_1,\ldots,x_d\}$.
For a subspace $V\subseteq H$, $P_V$
denotes the orthogonal projection onto $V$.
Write
\begin{equation*}
  V_i=\operatorname{span}\{x_\ell:\ell\in S_i\},
  \qquad
  r_t=f^\star-f_t,
  \qquad
  \Delta=\cE_0-\cE_M.
\end{equation*}
Here $r_t$ is the residual of $A_t$ against the global predictor, and $\Delta$ is the drop of the excess error over the block.
Every agent's prediction is a linear combination of inputs in $H$, so by
induction over the DAG it lies in $H$; in particular $r_t\in H$.
\Cref{lem:prelim-ls-projection} applied to the global predictor $f^\star$ gives
$\E[u(f^\star-Y)]=0$ for every $u\in H$, so
\begin{align*}
  \cE_t
  &=\MSE(f_t)-\MSE(f^\star) \\
  &=\E[(f_t-f^\star)^2]+2\E[(f_t-f^\star)(f^\star-Y)] \\
  &=\|r_t\|_2^2,
\end{align*}
where the second term vanishes by applying the orthogonality to
$u=f_t-f^\star$.

We check the hypotheses of \Cref{lem:block-residual-coverage} for
$r_0,\ldots,r_M$. First, by \Cref{lem:prelim-ls-projection}, the residuals
$Y-f_i$ and $Y-f^\star$ are orthogonal to every $v\in V_i$. Subtracting the two
orthogonality relations shows that $r_i=f^\star-f_i$ is orthogonal to $V_i$,
that is $P_{V_i}(r_i)=0$. Second, \Cref{eq:path-improvement} of
\Cref{lem:prelim-mse-decomposition} controls the steps:
$\|r_i-r_{i-1}\|_2^2=\|f_i-f_{i-1}\|_2^2=\cE_{i-1}-\cE_i$, so summing
telescopes to
\begin{equation*}
  \sum_{i=1}^M\|r_i-r_{i-1}\|_2^2=\cE_0-\cE_M=\Delta.
\end{equation*}

By coverage, $V_1+\cdots+V_M=H$, so \Cref{lem:coverage-lower-bound} gives a
constant $\lambda>0$ depending only on $\cD$ and $(S_1,\ldots,S_M)$, and
\Cref{lem:block-residual-coverage} applies. Since $r_0\in H$, we have
$P_H(r_0)=r_0$ and $\|r_0\|_2^2=\cE_0$, so its conclusion reads
$\lambda\cE_0\le2(M^2+1)\Delta$. Hence
\begin{equation*}
  \cE_M=\cE_0-\Delta\le\Bigl(1-\frac{\lambda}{2(M^2+1)}\Bigr)\cE_0,
\end{equation*}
and $q:=1-\lambda/(2(M^2+1))\in[0,1)$ depends only on $\cD$ and
$(S_1,\ldots,S_M)$.
\end{proof}

There are only finitely many possible tuples $(S_1,\ldots,S_M)$ in the lemma above. Thus, we can take the maximum over all tuples to obtain a global contraction factor.

\fixedInstanceRestate*

\begin{proof}
Fix $s$ with $1\le s\le D-M$. The sub-path $A_{s+1}\to\cdots\to A_{s+M}$
is a block of $M$ consecutive agents on the original path, so its feature
subsets cover $[d]$ by $M$-coverage, and its first agent has the on-path
predecessor $A_s$. Apply
\Cref{lem:fixed-block-contraction} to this sub-path with starting
predictor $f_s$: there is $q(S_{s+1},\ldots,S_{s+M})\in[0,1)$, depending
only on $\cD$ and $(S_{s+1},\ldots,S_{s+M})$, with
\begin{equation*}
  \cE_{s+M}\le q(S_{s+1},\ldots,S_{s+M})\,\cE_s.
\end{equation*}
The tuple $(S_{s+1},\ldots,S_{s+M})$ takes at most $2^{dM}$ values, so
$q=\max q(S_{s+1},\ldots,S_{s+M})\in[0,1)$ over the finitely many feasible
tuples; $q$ depends only on $\cD$ and $M$, and $\cE_{s+M}\le q\,\cE_s$ for
every $s$ with $1\le s\le D-M$.

By \Cref{eq:path-improvement} of \Cref{lem:prelim-mse-decomposition},
$\cE_t$ is non-increasing along the path. Iterating the per-block
contraction from $s=1$ over $\lfloor (D-1)/M\rfloor$ full blocks and using
monotonicity over the remaining steps gives
$\cE_D\le\cE_1\,q^{\lfloor (D-1)/M\rfloor}$.
\end{proof}

\section{Classification Results}
\label{sec:classification-results}

We now state the matching results for binary classification. The protocol is the logit-passing protocol from the preliminaries. For the upper bound, we use the same idea as in \Cref{thm:regression-improved-upper} to give a $O(M^2/D)$ upper bound, which improves upon the $O(M/\sqrt{D})$ bound from \citet{bateni2026networked}. For the lower bounds, we prove a Gaussian transfer result, which allows us to transfer the results directly from the regression model into classification. This includes a tighter analysis of the cyclic example (previously done by \citet{bateni2026networked}), and a $\Omega(M^2/D)$ lower bound when $D \ge M^2$ along with a constant lower bound for $D < M^2$. The latter example still depends on $D$, so we show analogous results to \Cref{thm:no-fixed-all-horizon,thm:fixed-distribution-geometric}.

\subsection{Upper Bound}

For the rest of this subsection, we assume the following conditions, which mirror those of the regression setting, and the ones used by \citet{bateni2026networked}:
\begin{itemize}
\item The global BCE minimizer $z_\star(x)=\sum_{\ell=1}^d w_\ell x_\ell$ satisfies $\sum_\ell |w_\ell|\le B_\star$.
\item Each feature satisfies $\E[x_\ell^2]\le B_X^2$.
\item The protocol minimizers are attained at finite coefficients.
\end{itemize}

Consider an $M$-covered path of agents $A_1 \to A_2 \to \cdots \to A_n$. By \Cref{lem:prelim-bce-kl}, the losses are non-increasing along the path, meaning that $\cL(z_1) \ge \cL(z_2) \ge \cdots \ge \cL(z_n)$. As in the regression setting, if the loss does not decrease significantly over a block of agents that sees every feature, then the excess loss of the last agent in the block is small. The following lemma is the classification analogue of \Cref{lem:regression-block-residual}. \citet[Lemma 3.6]{bateni2026networked} also shows a similar result. We provide the proof for completeness.

\begin{lemma}
\label{lem:classification-block-residual}
Consider any path $A_1 \to A_2 \to \dots \to A_n$ and a block of $k$ consecutive agents indexed by $[a+1, b]$ that sees every raw feature at least
once, where $b=a+k$.
If $\varepsilon\ge\cL(z_a)-\cL(z_b)$, then
\begin{equation*}
  \cL(z_b)-\cL(z_\star)
  \le
  B_\star B_X\sqrt{\frac{k\varepsilon}{2}}.
\end{equation*}
\end{lemma}

\begin{proof}
Fix a feature $x_\ell$, and let agent $A_j$ in the block see it. BCE orthogonality from \Cref{lem:prelim-bce-orthogonality} gives $\E[x_\ell(\sigma(z_j)-Y)]=0$. For consecutive agents on the path, \Cref{eq:bce-kl-identities} gives $\|\sigma(z_i)-\sigma(z_{i-1})\|_2^2\le\frac12(\cL(z_{i-1})-\cL(z_i))$, so summing over the block,
\begin{equation*}
  \sum_{i=a+1}^b \|\sigma(z_i)-\sigma(z_{i-1})\|_2^2
  \le
  \frac12\bigl(\cL(z_a)-\cL(z_b)\bigr)
  \le
  \frac{\varepsilon}{2}.
\end{equation*}
Since agent $A_j$ sees $x_\ell$, its orthogonality $\E[x_\ell(\sigma(z_j)-Y)]=0$ lets us
replace the target $Y$ with $\sigma(z_j)$ in the following equation:
\begin{equation*}
  \E[x_\ell(\sigma(z_b)-Y)]
  =
  \E[x_\ell(\sigma(z_b)-\sigma(z_j))]
  +
  \E[x_\ell(\sigma(z_j)-Y)]
  =
  \E[x_\ell(\sigma(z_b)-\sigma(z_j))].
\end{equation*}
Write $\sigma(z_b)-\sigma(z_j)$ as $\sigma(z_b)-\sigma(z_j)=\sum_{i=j+1}^b(\sigma(z_i)-\sigma(z_{i-1}))$. The sum has at most
$k$ terms, since $j\ge a+1$ and $b=a+k$. The triangle inequality followed by Cauchy--Schwarz across these terms gives
\begin{equation*}
  \|\sigma(z_b)-\sigma(z_j)\|_2
  \le
  \sum_{i=j+1}^b\|\sigma(z_i)-\sigma(z_{i-1})\|_2
  \le
  \sqrt{k}\left(\sum_{i=j+1}^b\|\sigma(z_i)-\sigma(z_{i-1})\|_2^2\right)^{1/2}
  \le
  \sqrt{\frac{k\varepsilon}{2}},
\end{equation*}
where the last step uses the bound $\sum_{i=a+1}^b\|\sigma(z_i)-\sigma(z_{i-1})\|_2^2\le\varepsilon/2$
from above. A final Cauchy--Schwarz over the feature, with $\E[x_\ell^2]\le B_X^2$,
then yields
\begin{equation*}
  |\E[x_\ell(\sigma(z_b)-Y)]|
  =
  |\E[x_\ell(\sigma(z_b)-\sigma(z_j))]|
  \le
  \sqrt{\E[x_\ell^2]}\,\|\sigma(z_b)-\sigma(z_j)\|_2
  \le
  B_X\sqrt{\frac{k\varepsilon}{2}}.
\end{equation*}
The bound just derived holds for every feature $\ell$, since the block sees each
raw feature at least once. To turn these per-feature bounds into a loss bound,
apply the comparator inequality of \Cref{lem:prelim-classification-comparator} with $z_g=z_\star$: since $z_b$ minimizes BCE over its inputs,
\begin{equation*}
  \cL(z_b)-\cL(z_\star)
  \le
  |\E[(\sigma(z_b)-Y)z_\star]|.
\end{equation*}
Finally, expand $z_\star=\sum_\ell w_\ell x_\ell$ and use the per-feature bound together with
$\sum_\ell|w_\ell|\le B_\star$,
\begin{equation*}
  |\E[(\sigma(z_b)-Y)z_\star]|
  \le
  \sum_\ell|w_\ell|\,|\E[x_\ell(\sigma(z_b)-Y)]|
  \le
  B_\star B_X\sqrt{\frac{k\varepsilon}{2}},
\end{equation*}
which gives $\cL(z_b)-\cL(z_\star)\le B_\star B_X\sqrt{k\varepsilon/2}$.
\end{proof}

We now state the classification analogue of \Cref{lem:regression-suffix-improvement}. The proof for this lemma is almost identical to the proof of \Cref{lem:regression-suffix-improvement}.

\begin{lemma}[A good suffix from a good prefix]
\label{lem:classification-suffix-improvement}
Take an $M$-covered path $A_1 \to A_2 \to \dots \to A_n$ and an agent $A_s$ on it with
$\cL(z_s) - \cL(z_\star)\le\delta$. Let $L = n-s$ and consider the suffix $A_{s+1}\to\dots\to A_{s+L}$. If $L \ge 2M$, then
\begin{equation*}
  \cL(z_{n})-\cL(z_\star)
  \le
  B_\star B_X\,M\sqrt{\frac{\delta}{L}}.
\end{equation*}
\end{lemma}

\begin{proof}
Split the suffix into $K=\lfloor L/M\rfloor$ full blocks. Since $L\ge2M$,
$K\ge L/(2M)$. The total loss drop over these blocks is at most
$\cL(z_s)-\cL(z_\star)\le\delta$, so by the pigeonhole principle, some block has drop
$\varepsilon\le\delta/K\le2M\delta/L$. Applying
\Cref{lem:classification-block-residual} on this block with $k=M$ gives, at the end $q$ of that
block,
\begin{equation*}
  \cL(z_q)-\cL(z_\star)
  \le
  B_\star B_X\sqrt{\frac{M^2\delta}{L}}
  =
  B_\star B_X\,M\sqrt{\frac{\delta}{L}}.
\end{equation*}
By \Cref{lem:prelim-bce-kl}, the loss is non-increasing along the path, so $\cL(z_n)\le\cL(z_q)$.
\end{proof}

Finally, we prove the upper bound theorem. Its proof is also almost identical to \Cref{thm:regression-improved-upper}.

\classificationUpperRestatement*

\begin{proof}
Set $c_\star=B_\star B_X$, the constant of
\Cref{lem:classification-suffix-improvement}, and $C=\max\{4\log2,\,6c_\star^2\}$. We
prove that $\cL(z_D)-\cL(z_\star)\le CM^2/D$ for every $M$-covered path of
depth $D$, arguing by induction on $D$.

Since the loss is non-increasing along the path by \Cref{lem:prelim-bce-kl} and BCE is
nonnegative, $\cL(z_D)-\cL(z_\star)\le\cL(z_D)\le\cL(z_1)\le\cL(0)=\log2$, using that
the zero logit is feasible for the first agent. This establishes the
bound for short paths: if $D<4M$, then $M^2/D>M/4\ge1/4$, hence
\begin{equation*}
  \cL(z_D)-\cL(z_\star)
  \le
  \log2
  \le
  \frac{C}{4}
  <
  C\frac{M^2}{D}.
\end{equation*}

Now suppose $D\ge4M$, and split the path at $s=\lfloor D/2\rfloor$ into a prefix
$A_1\to\dots\to A_s$ and a suffix $A_{s+1}\to\dots\to A_D$ of length $L=D-s$.
Both are $M$-covered, being sub-paths of an $M$-covered path. We apply the induction hypothesis on the prefix. Since
$s=\lfloor D/2\rfloor\ge(D-1)/2\ge D/3$ (using $D\ge3$),
\begin{equation*}
  \cL(z_s)-\cL(z_\star)
  \le
  \frac{CM^2}{s}
  \le
  \frac{3CM^2}{D}
  =:\delta.
\end{equation*}
The suffix has length $L=D-s\ge D/2\ge2M$, so
\Cref{lem:classification-suffix-improvement} applies with this $\delta$. Substituting
$\delta=3CM^2/D$ and $L\ge D/2$ then gives
\begin{equation*}
  \cL(z_D)-\cL(z_\star)
  \le
  c_\star M\sqrt{\frac{\delta}{L}}
  \le
  c_\star M\sqrt{\frac{3CM^2/D}{D/2}}
  =
  c_\star\sqrt{6C}\,\frac{M^2}{D}.
\end{equation*}
Finally, $C\ge6c_\star^2$ implies $c_\star\sqrt{6C}\le C$, so the right-hand side
is at most $CM^2/D$, which completes the induction.
\end{proof}

\subsection{Gaussian Transfer}

\citet{bateni2026networked} show a lower bound of $\Omega(M/D)$ for the
classification setting. They use the same cyclic example and apply $\sigma$ to
the label to obtain probabilities. Their analysis for this example does not
directly use the regression bound.

We instead take a shortcut. In this section, we show that any Gaussian instance
from the regression setting has the same asymptotic error if the agents instead
fit predictors for linear classification. The cyclic example given by
\citet{kearns2026networked} and the example we provide in
\Cref{sec:depth-dependent-obstruction} are both Gaussian, so their bounds
directly transfer to the classification setting.

\begin{lemma}[BCE loss and squared logit error]
\label{lem:classification-bce-logit-comparison}
Let $x_1, \dots, x_d$ be centered jointly Gaussian variables and
$G = w^T x$ be a linear combination of them. Let $Y\in\{0,1\}$ satisfy
$\Pr(Y=1\mid x) = \sigma(G)$.
If $z$ is a linear combination of the features $x$, then
\begin{equation*}
  0\le \cL(z)-\cL(G)\le \frac18\E[(z-G)^2].
\end{equation*}
Moreover, for each $0\le B<\infty$ there is a constant $\kappa_B>0$ such that, if
$\E[G^2],\E[z^2]\le B$,
\begin{equation*}
  \cL(z)-\cL(G)\ge \kappa_B\E[(z-G)^2].
\end{equation*}
\end{lemma}

\begin{proof}
Recall that $\phi(u)=\log(1+e^u)$, $\phi'(u)=\sigma(u)$, and
$\cL(u) = \E[\phi(u) - Yu]$. Since $G$ and $z$ are linear combinations of the
centered Gaussian vector $x$, the pair $(G,z)$ is centered jointly Gaussian.
Also $G$ and $z$ are functions of $x$, so the assumption
$\Pr(Y=1\mid x)=\sigma(G)$ gives
$\E[Y\mid G,z]=\sigma(G)=\phi'(G)$. Hence
\begin{equation*}
  \E[Y(z-G)]
  =
  \E[\E[Y\mid G,z](z-G)]
  =
  \E[\phi'(G)(z-G)].
\end{equation*}
Thus
\begin{equation*}
  \begin{aligned}
  \cL(z)-\cL(G)
  =
  \E[\phi(z)-\phi(G)-Y(z-G)]
  =
  \E[\phi(z)-\phi(G)-\phi'(G)(z-G)].
  \end{aligned}
\end{equation*}
For real numbers $a$ and $b$, Taylor's remainder formula gives some point $\xi$
between $a$ and $b$ such that
\begin{equation}
  \label{eq:phi-taylor-remainder}
  \phi(b)-\phi(a)-\phi'(a)(b-a)
  =
  \frac12\phi''(\xi)(b-a)^2.
\end{equation}
Also $\phi''(u)=\sigma(u)(1-\sigma(u))$, so
$0\le\phi''(u)\le1/4$. Combining the above equalities with $a=G$ and $b=z$
gives
\begin{equation*}
  0\le \cL(z)-\cL(G)\le \frac18\E[(z-G)^2].
\end{equation*}

It remains to prove the lower bound under $\E[G^2],\E[z^2]\le B$. If
$\E[(z-G)^2]=0$, then $z=G$ almost surely and the claim is trivial.

Choose $R_B>0$ with $R_B^2\ge48B$. By Chebyshev's inequality, every centered
random variable $W$ with $\E[W^2]\le B$ satisfies
$\Pr(|W|>R_B)\le1/48$. Define the event
$A=\{|G|\le R_B,\ |z|\le R_B\}$. We first lower-bound the loss on $A$. For
$|u|\le R_B$,
\begin{equation*}
  \phi''(u)
  =
  \frac{1}{(e^{u/2}+e^{-u/2})^2}
  \ge
  \frac{e^{-R_B}}{4}.
\end{equation*}
Here the last step uses $e^{u/2}+e^{-u/2}\le2e^{R_B/2}$. On $A$, the
intermediate point of \Cref{eq:phi-taylor-remainder} also lies in
$[-R_B,R_B]$, so the equation gives
\begin{equation*}
  \phi(z)-\phi(G)-\phi'(G)(z-G)
  \ge
  \frac18 e^{-R_B}(z-G)^2.
\end{equation*}
The same formula, namely the left-hand side of
\Cref{eq:phi-taylor-remainder}, is nonnegative everywhere, so we can discard
the contribution from $A^c$ and get
\begin{equation}
  \label{eq:classification-good-event-loss}
  \cL(z)-\cL(G)
  \ge
  \frac18 e^{-R_B}\E[(z-G)^2\ind{A}].
\end{equation}

It remains to show that $A$ contains a fixed fraction of $\E[(z-G)^2]$. Since
$(G,z)$ is centered jointly Gaussian, the difference $z-G$ is centered
Gaussian, and hence
$\E[(z-G)^4]=3\E[(z-G)^2]^2$. The choice of $R_B$ applies to both $G$ and $z$.
For an event $E$, Cauchy--Schwarz gives
\begin{equation*}
\E[(z-G)^2\ind{E}]\le
\sqrt{\E[(z-G)^4]\Pr(E)}.
\end{equation*}
Applying this with
$E=\{|G|>R_B\}$ and then with $E=\{|z|>R_B\}$ gives
\begin{align*}
  \E[(z-G)^2\ind{|G|>R_B}]
  &\le
  \sqrt{\E[(z-G)^4]\,\Pr(|G|>R_B)} \\
  &\le
  \sqrt{3\E[(z-G)^2]^2/48}
  =
  \frac14\E[(z-G)^2], \\
  \E[(z-G)^2\ind{|z|>R_B}]
  &\le
  \sqrt{\E[(z-G)^4]\,\Pr(|z|>R_B)} \\
  &\le
  \sqrt{3\E[(z-G)^2]^2/48}
  =
  \frac14\E[(z-G)^2].
\end{align*}
Since $A^c=\{|G|>R_B\}\cup\{|z|>R_B\}$,
\begin{align*}
  \E[(z-G)^2\ind{A}]
  &=
  \E[(z-G)^2]-\E[(z-G)^2\ind{A^c}]
  \\
  &\ge
  \E[(z-G)^2]
  -\E[(z-G)^2\ind{|G|>R_B}]
  -\E[(z-G)^2\ind{|z|>R_B}]
  \\
  &\ge
  \frac12\E[(z-G)^2].
\end{align*}
Substituting this into \Cref{eq:classification-good-event-loss} gives
\begin{align*}
  \cL(z)-\cL(G)
  &\ge
  \frac18 e^{-R_B}\E[(z-G)^2\ind{A}] \\
  &\ge
  \frac{e^{-R_B}}{16}\E[(z-G)^2].
\end{align*}
Thus the claim holds with $\kappa_B=e^{-R_B}/16$.
\end{proof}

We remark that the idea of restricting a Gaussian variable to a centered interval in order to bound the sigmoid function's derivative also shows up in the proof of \cite{bateni2026networked}, though their reasoning considers a specific Gaussian variable.

The next lemma shows that under the same Gaussian inputs, the BCE minimizer $z$ and the least-squares minimizer $f$ are proportional, $z=cf$. Similar proportionality results are known~\citep{brillinger1982generalized,erdogdu2016scaled}. We include the proof for completeness and further show that $c\in[0,1]$.

\begin{lemma}[Gaussian logistic regression keeps the least-squares direction]
\label{lem:classification-gaussian-projection}
Let $x_1, \dots, x_d$ be centered jointly Gaussian variables and
$G = w^T x$ be a linear combination of them. Let $Y\in\{0,1\}$ satisfy
$\Pr(Y=1\mid x) = \sigma(G)$.

Let $u_1,\ldots,u_q$ be linear combinations of $x$. Let $f$ be the linear
least-squares minimizer for label $G$ over $u_1, \ldots, u_q$, and let $z$ be
the BCE minimizer for label $Y$. Then
\begin{equation*}
z=cf
\end{equation*}
for some $c\in[0,1]$, with $c>0$ whenever $f\ne0$.
\end{lemma}

\begin{proof}
Let $V = \operatorname{span}(u_1, \ldots, u_q)$. By
\Cref{lem:prelim-ls-projection}, applied with label $G$, $f$ is the projection
of $G$ onto $V$, and the residual $G-f$ is orthogonal to every element of $V$.

First suppose $f=0$. Since the residual $G-f$ is orthogonal to $V$, $G$ is
independent of $V$. For any $h\in V$, the variable $h$ is centered, and the
symmetry of the centered Gaussian $G$ gives
$\E[\sigma(G)]=1/2$. Hence
\begin{equation*}
  \E[Yh]
  =
  \E[\sigma(G)h]
  =
  \E[\sigma(G)]\E[h]
  =
  0.
\end{equation*}
Thus
\begin{equation*}
  \cL(h) = \E[\phi(h) - Yh] = \E[\phi(h)].
\end{equation*}
By \Cref{lem:prelim-bce-convexity}, $\phi$ is strictly convex, so Jensen's
inequality gives
\begin{equation*}
  \E[\phi(h)] \ge \phi(\E[h]) = \phi(0) = \log 2 = \cL(0).
\end{equation*}
Equality holds only when $h=0$ almost surely, so the BCE minimizer is $z=0$.

Now suppose $f\ne0$. We first show that the BCE minimizer cannot use any
direction orthogonal to $f$. Write
\begin{equation*}
  z=cf+r,
\end{equation*}
where $r\in V$ is orthogonal to $f$. Since $G-f$ is orthogonal to
every element of $V$, the variable $r$ is orthogonal to $G-f$. Since $r$ is also
orthogonal to $f$, it is orthogonal to $G=(G-f)+f$. Joint Gaussianity then
implies that $r$ is independent of $(G,f)$. Also
$\Pr(Y=1\mid G,f,r)=\sigma(G)$, so $r$ is independent of $Y$ after conditioning
on $(G,f)$. Thus $r$ is independent of $(G,f,Y)$ and remains centered after we
condition on these variables. Therefore, Jensen gives
\begin{align*}
  \E[\phi(cf+r)-Y(cf+r)\mid G,f,Y]
  &\ge
  \phi(cf+\E[r\mid G,f,Y])-Ycf \\
  &=
  \phi(cf)-Ycf.
\end{align*}
Taking expectations shows that $\cL(z)\ge\cL(cf)$. Since $z$ is a minimizer and
$cf$ is feasible, equality must hold. Moreover,
because $\phi$ is strictly convex, equality in Jensen can hold only if $r$ is
almost surely constant after conditioning on $(G,f,Y)$. But $r$ is independent
of $(G,f,Y)$ and centered, so this means $r=0$ almost surely. Thus any finite
BCE minimizer has the form $z=cf$.

It remains to locate $c$. Recall that the residual $G-f$ is independent of
$f$. Define $\psi(s)=\E[\sigma(s+G-f)]$, with the expectation over the residual
$G-f$.
Along the scalar line, define
\begin{equation*}
  F(a)=\cL(af)=\E[\phi(af)-\sigma(G)af],
\end{equation*}
and
\begin{equation*}
  F'(a)=\E[f(\sigma(af)-\psi(f))].
\end{equation*}
The function $\psi$ is strictly increasing and $\psi(0)=1/2$. Therefore, when
$f>0$, we have $\psi(f)>1/2$, and when $f<0$, we have $\psi(f)<1/2$. In both
cases,
\begin{equation*}
  f(1/2-\psi(f))<0
\end{equation*}
whenever $f\ne0$. Since $f$ is not almost surely zero, taking expectations gives
$F'(0)=\E[f(1/2-\psi(f))]<0$. Thus the loss decreases as we move to the right
from $0$. By \Cref{lem:prelim-bce-convexity}, $\cL$ is convex, so $F$ is convex;
thus no minimizer of $F$ can lie at or to the left of $0$. Since $z=cf$, this
gives $c>0$.

We now show $c\le1$. For $s>0$ and every real $n$,
\begin{equation*}
  \frac{\sigma(s+n)+\sigma(s-n)}{2}
  =
  \frac12+\frac{\sinh s}{2(\cosh s+\cosh n)}
  \le
  \frac12+\frac{\sinh s}{2(\cosh s+1)}
  =
  \sigma(s).
\end{equation*}
The residual $G-f$ is a centered Gaussian, so $G-f$ and $f-G$ have the same
distribution. For $s>0$, we can plug the random value $G-f$ into the inequality
above and take expectation:
\begin{equation*}
  \E\left[\frac{\sigma(s+G-f)+\sigma(s+f-G)}{2}\right]
  \le
  \sigma(s).
\end{equation*}
Since $G-f$ and $f-G$ have the same distribution, the left side is
$\E[\sigma(s+G-f)]=\psi(s)$. Thus $\psi(s)\le\sigma(s)$ for $s>0$. By symmetry,
$\psi(s)\ge\sigma(s)$ for $s<0$. Hence
$f(\sigma(f)-\psi(f))\ge0$ for every value of $f$, and so $F'(1)\ge0$. Since
$F$ is convex, any minimizer of $F$ is at most $1$. Therefore $c\in(0,1]$,
concluding the proof.
\end{proof}

\classificationTransferRestatement*

\begin{proof}
Suppose agents learn in order $A_1,\ldots,A_N$. We first prove the
$z_t = c_t f_t$ part of the claim. Induct on $t$. For $t=1$, agent $A_1$
receives the same raw features in both the least-squares and classification
settings. Thus \Cref{lem:classification-gaussian-projection} applies and proves
the $z_1 = c_1 f_1$ part of the claim with $c_1\in[0,1]$.
The same lemma gives $c_1>0$ whenever $f_1\ne0$.

Now fix $t>1$ and assume the claim has been proved for all earlier agents. Let
$V_t^{\mathrm{ls}}$ be the span of the raw features of $A_t$ and the parent
least-squares predictions $f_i$ for $i\in\Pa(t)$. Let $V_t^{\mathrm{bce}}$ be
the span of the same raw features and the parent logits $z_i$ for
$i\in\Pa(t)$. By the induction hypothesis, $z_i=c_if_i$ for each parent. If
$f_i=0$, then $z_i=0$. If $f_i\ne0$, then $c_i>0$, so $z_i$ spans the same line
as $f_i$. Thus $V_t^{\mathrm{ls}}=V_t^{\mathrm{bce}}$; call this common space
$V_t$.

The least-squares predictor $f_t$ minimizes squared loss over $V_t$, and the
classification logit $z_t$ is the finite BCE minimizer over $V_t$. The
raw features are coordinates of $x$, and the parent predictions are linear
combinations of earlier inputs, so all elements of $V_t$ are linear
combinations of the Gaussian features $x$. Hence
\Cref{lem:classification-gaussian-projection} applies and gives
$z_t=c_tf_t$ for some $c_t\in[0,1]$, with $c_t>0$ whenever $f_t\ne0$. This
completes the induction.

It remains to compare the losses. Fix $t$. Again let $V_t$ denote the span of
the inputs of $A_t$. By \Cref{lem:prelim-ls-projection}, $f_t$ is the projection
of $G$ onto $V_t$, so $G-f_t$ is orthogonal to $f_t$. Hence
\begin{align*}
  \E[G^2]
  &=
  \E[(f_t+(G-f_t))^2] \\
  &=
  \E[f_t^2]+2\E[f_t(G-f_t)]+\E[(G-f_t)^2] \\
  &=
  \E[f_t^2]+\E[(G-f_t)^2].
\end{align*}
Thus $\E[f_t^2]\le\E[G^2]$. Since $z_t=c_tf_t$ with $c_t\in[0,1]$, and since
$G$ is centered,
\begin{equation*}
  \E[z_t^2]
  =
  c_t^2\E[f_t^2]
  \le
  \E[G^2]
  =
  \operatorname{Var}(G)
  \le
  B.
\end{equation*}
The lower bound in \Cref{lem:classification-bce-logit-comparison} therefore
applies to the pair $(G,z_t)$. Also, $z_t\in V_t$ and $f_t$ is the projection of
$G$ onto $V_t$, so
\begin{equation*}
  \E[(G-z_t)^2]\ge \E[(G-f_t)^2]=E_t.
\end{equation*}
Thus
\begin{equation*}
  \cL(z_t)-\cL(G)
  \ge
  \kappa_B\E[(G-z_t)^2]
  \ge
  \kappa_B E_t.
\end{equation*}

For the upper bound, $f_t$ is a feasible logit for the BCE problem at agent
$A_t$, because the two input spaces are the same. Since $z_t$ minimizes BCE over
this space,
\begin{equation*}
  \cL(z_t)\le \cL(f_t).
\end{equation*}
Applying the upper bound in
\Cref{lem:classification-bce-logit-comparison} to $f_t$ gives
\begin{equation*}
  \cL(z_t)-\cL(G)
  \le
  \cL(f_t)-\cL(G)
  \le
  \frac18\E[(G-f_t)^2]
  =
  \frac18 E_t. \qedhere
\end{equation*}
\end{proof}

\subsection{Lower Bounds}

We next apply the Gaussian transfer theorem to the regression cyclic and
depth-dependent lower bounds. In both cases the true logit $G$ is an exact linear
combination of the features, and the label is drawn with conditional mean
$\sigma(G)$, so \Cref{thm:classification-sequential-transfer} applies.

\classificationCyclicRestatement*

\begin{proof}
The least-squares error after pass $p$ is at least $1/(48\sqrt p)$ by
\Cref{thm:kearns-cyclic-lower-bound}. Here
$\operatorname{Var}(G)=1$, so
\Cref{thm:classification-sequential-transfer} gives the result.
\end{proof}

\classificationLongDepthRestatement*

\begin{proof}
Use the regression instance from \Cref{thm:long-depth}, with the same Gaussian
logit $G=X_\star+\rho Z_0$ and Bernoulli mean $\sigma(G)$. The coverage, exact
predictor, coefficient bound, and second-moment bound are proved there. That
theorem also gives
\begin{equation*}
  E_D
  \ge
  \frac{1}{1280\pi^2}\frac{M^2}{D}
\end{equation*}
for the matching least-squares path. In this construction
$\operatorname{Var}(G)=1+\rho^2\le2$. Hence
\Cref{thm:classification-sequential-transfer} gives the claim with
$c_{\mathrm{ld}}=\kappa_2/(1280\pi^2)$.
\end{proof}

\classificationConstantRestatement*

\begin{proof}
Apply \Cref{thm:classification-long-depth} with target depth $D=M^2$. By
\Cref{lem:prelim-bce-kl}, the loss is non-increasing along the path. Thus, for every $1\le t\le M^2$,
$\cL(z_t)-\cL(G)\ge\cL(z_{M^2})-\cL(G)\ge c_{\mathrm{ld}}$.
Take $c_{\mathrm{quad}}=c_{\mathrm{ld}}$.
\end{proof}

\subsection{The Fixed-Distribution Obstruction}

The depth-dependent lower bound in \Cref{sec:depth-dependent-obstruction} used
a distribution that changes with $D$, since the target radius $\rho$ shrank with the
depth. As in regression, we show that this is unavoidable.

The argument parallels the regression proof of \Cref{lem:fixed-block-contraction}.
Let $H=\operatorname{span}\{x_1,\ldots,x_d\}$, which is finite-dimensional. Every
logit is a linear combination of raw features and parent logits, so every logit $z_t$, along with the global BCE minimizer $z_\star$, lies in $H$. The zero logit is always feasible, so
$\cL(z_t)\le\cL(0)=\log2$ for every agent, so every predictor lies in the
sublevel set
\begin{equation}
K=\{z\in H:\cL(z)\le\cL(0)\}.
\end{equation}
At agent $A_t$ write
$V_t=\operatorname{span}\{x_\ell:\ell\in S_t\}$ for the span of its raw features.
The agent minimizes BCE over a space that contains $V_t$ and other logits from $\Pa(t)$.

We write $\cE_t=\cL(z_t)-\cL(z_\star)$ for the excess loss and consider a path that is $M$-covered. This means that every $M$ consecutive raw-feature spans $V_t$ sum to $H$.

The regression proof rested on the identity $\cE_t=\|f^\star-f_t\|_2^2$, which
turns excess error into a squared distance. Cross-entropy is not a squared norm,
but on the bounded set $K$ it is strongly convex, and this lets us
compare the excess loss with the squared logit distance to $z_\star$. We record
this comparison first and the block argument then follows the regression one.

\begin{lemma}[BCE excess loss and logit distance]
\label{lem:classification-strong-convexity}
Assume $x_1,\ldots,x_d$ have bounded second moments. Suppose $H\ne\{0\}$, the BCE minimum over $H$ is attained at $z_\star$, and
$\cL(0)>\cL(z_\star)$. Then $K$ is compact, and there is a constant
$\mu\in(0,\tfrac14]$, depending only on $\cD$, such that every $z\in K$ satisfies
\begin{equation*}
  \E[(\sigma(z)-\sigma(z_\star))(z-z_\star)]\ge\mu\|z-z_\star\|_2^2
  \qquad\text{and}\qquad
  \cL(z)-\cL(z_\star)\le\tfrac18\|z-z_\star\|_2^2 .
\end{equation*}
\end{lemma}

\begin{proof}
By \Cref{lem:prelim-bce-convexity}, $\cL$ is convex, and therefore $K$ is convex.

We first show $K$ is compact. Since $\cL$ is continuous, $K$ is closed; as it
lives in the finite-dimensional $H$, it is enough to bound it.

The main step is a linear lower bound on $\cL$. Let $\eta=\E[Y\mid x_1,\ldots,x_d]$,
and for $z\in H$ split $z=z_+-z_-$ with $z_+=\max\{z,0\}$ and $z_-=\max\{-z,0\}$.
Each $z\in H$ is a function of $x$, so $\E[Yz]=\E[\eta z]$, and $\phi(u)=\log(1+e^u)\ge u_+$ gives
\begin{equation*}
  \cL(z)=\E[\phi(z)-Yz]\ge\E[z_+-\eta z]=\E[(1-\eta)z_++\eta z_-]=:r(z)\ge0 .
\end{equation*}
The functional $r$ is continuous and positively homogeneous by definition: $r(az)=a\,r(z)$ for $a\ge0$.

We claim $r(h)>0$ for every $h\ne0$. Suppose instead $r(h)=0$. Its two terms are
nonnegative, so $(1-\eta)h_+=0$ and $\eta h_-=0$ almost surely, meaning $\eta=1$
where $h>0$ and $\eta=0$ where $h<0$. Since $z_\star$ minimizes $\cL$ over $H$ and
$h\in H$, \Cref{lem:prelim-bce-orthogonality} gives $\E[(\sigma(z_\star)-Y)h]=0$,
and as $h$ is a function of $x$ this equals $\E[(\sigma(z_\star)-\eta)h]$. But
$\sigma(z_\star)\in(0,1)$, so $\sigma(z_\star)-\eta<0$ where $h>0$ and
$\sigma(z_\star)-\eta>0$ where $h<0$; hence $(\sigma(z_\star)-\eta)h<0$ on
$\{h\ne0\}$, a set of positive probability, forcing $\E[(\sigma(z_\star)-\eta)h]<0$.
This contradiction proves the claim.

Thus $r>0$ on the unit sphere of $H$, which is compact, so $r\ge\rho_0$ there for
some $\rho_0>0$. By homogeneity $r(z)\ge\rho_0\|z\|_2$ for all $z\in H$, so
$\cL(z)\ge\rho_0\|z\|_2$. Any $z\in K$ then obeys $\rho_0\|z\|_2\le\cL(z)\le\log2$,
that is $\|z\|_2\le(\log2)/\rho_0$. So $K$ is closed and bounded in the
finite-dimensional $H$, and therefore compact.

Write $\sigma'(w)=\sigma(w)(1-\sigma(w))\in(0,\tfrac14]$ for the derivative
of the sigmoid. For $w\in K$ and $g\in H$, the value $\E[\sigma'(w)g^2]$ is positive
whenever $g\ne0$, and it varies continuously with $(w,g)$ over the
finite-dimensional space $H\times H$. Since $\{(w,g):w\in K,\ \|g\|_2=1\}$ is
compact, $\E[\sigma'(w)g^2]$ attains a positive minimum $\mu$ there, and
$\mu\le\tfrac14$ because $\sigma'\le\tfrac14$. Scaling in $g$ gives
$\E[\sigma'(w)g^2]\ge\mu\|g\|_2^2$ for all $w\in K$ and $g\in H$.

Now fix $z\in K$ and set $g=z-z_\star$; the segment from $z_\star$ to $z$ stays in
$K$, so $z_\star+\theta g\in K$ for $0\le\theta\le1$. The fundamental theorem of
calculus, applied to $\theta\mapsto\sigma(z_\star+\theta g)$, gives
$\sigma(z)-\sigma(z_\star)=\int_0^1\sigma'(z_\star+\theta g)\,g\,d\theta$ pointwise.
Multiplying by $g$ makes the integrand nonnegative, so expectation and integral
exchange, and
\begin{equation*}
  \E[(\sigma(z)-\sigma(z_\star))(z-z_\star)]
  =\int_0^1\E[\sigma'(z_\star+\theta g)g^2]\,d\theta
  \ge\mu\|g\|_2^2 ,
\end{equation*}
since each $z_\star+\theta g\in K$ gives $\E[\sigma'(z_\star+\theta g)g^2]\ge\mu\|g\|_2^2$.
For the second bound, $z_\star$ minimizes $\cL$ over $H\ni g$, so
\Cref{lem:prelim-bce-orthogonality} gives $\E[(\sigma(z_\star)-Y)g]=0$. Writing
$\phi(u)=\log(1+e^u)$, so $\phi'=\sigma$ and $\phi''=\sigma'$, this replaces $Y$ by
$\phi'(z_\star)$ in the loss gap:
\begin{equation*}
  \cL(z)-\cL(z_\star)=\E[\phi(z)-\phi(z_\star)-\phi'(z_\star)g].
\end{equation*}
Applying the second-order Taylor remainder as in \Cref{eq:phi-taylor-remainder}
with $a=z_\star$ and $b=z$ writes the integrand as $\tfrac12\phi''(\xi)g^2$ for
some $\xi$ between $z_\star$ and $z$, and $\phi''=\sigma'\le\tfrac14$ gives
$\cL(z)-\cL(z_\star)\le\tfrac18\|g\|_2^2$.
\end{proof}

With this comparison in hand, the block contraction again rests on
\Cref{lem:block-residual-coverage}, now applied to the probability residual
$\sigma(z_t)-\sigma(z_\star)$ in place of the regression residual $f^\star-f_t$.

\begin{lemma}[Per-block contraction]
\label{lem:classification-block-contraction}
Fix a distribution $\cD$ on $(x_1,\ldots,x_d,Y)$ with $Y\in\{0,1\}$, $d$ finite,
and BCE minimum over $H$ attained at $z_\star$. Assume $x_1,\ldots,x_d$ have bounded second moments. Consider a path
$A_0\to A_1\to\cdots\to A_M$ of agents on $\cD$ in any DAG, with raw-feature
spans $V_1,\ldots,V_M$ satisfying $V_1+\cdots+V_M=H$. There is a constant
$q\in[0,1)$, depending only on $\cD$ and $(S_1,\ldots,S_M)$, such that
$\cE_M\le q\,\cE_0$.
\end{lemma}

\begin{proof}
If $\cE_0=0$, the loss is non-increasing along the path, so $\cE_M\le\cE_0=0$ and any $q$ works.
If $H=\{0\}$ or $\cL(0)=\cL(z_\star)$, every logit already minimizes and all
$\cE_t=0$. So assume $\cE_0>0$, $H\ne\{0\}$, and $\cL(0)>\cL(z_\star)$, and let
$\mu$ be the constant of \Cref{lem:classification-strong-convexity}. Write
$p_t=\sigma(z_t)$, $p_\star=\sigma(z_\star)$, and set
\begin{equation*}
  r_t=p_t-p_\star,
  \qquad
  \Delta=\cE_0-\cE_M ,
\end{equation*}
the probability residual of $A_t$ and the loss drop over the block. For a
subspace $V$, let $P_V$ be the orthogonal projection onto $V$.

We check the hypotheses of \Cref{lem:block-residual-coverage} for
$r_0,\ldots,r_M$. The residual $r_i$ is orthogonal to the block features: both
$z_i$ and $z_\star$ minimize BCE over a space containing $V_i$, so
\Cref{lem:prelim-bce-orthogonality} gives $\E[(p_i-Y)v]=\E[(p_\star-Y)v]=0$ for
every $v\in V_i$. Subtracting, $\E[r_iv]=0$, that is $P_{V_i}(r_i)=0$. Each
step is small: for consecutive agents, \Cref{eq:bce-kl-identities} gives
$\|r_i-r_{i-1}\|_2^2=\|p_{i-1}-p_i\|_2^2\le\tfrac12(\cE_{i-1}-\cE_i)$, so
summing telescopes to
\begin{equation*}
  \sum_{i=1}^M\|r_i-r_{i-1}\|_2^2\le\tfrac12\Delta.
\end{equation*}
\Cref{lem:coverage-lower-bound} gives a constant $\lambda>0$ depending only on
$\cD$ and $(S_1,\ldots,S_M)$, so \Cref{lem:block-residual-coverage} gives
\begin{equation}
  \label{eq:classification-r0-vs-drop}
  \lambda\|P_H(r_0)\|_2^2\le2(M^2+1)\cdot\tfrac12\Delta=(M^2+1)\,\Delta.
\end{equation}

Finally, we bound $\|P_H(r_0)\|_2$ from below by the excess loss. The predecessor
logit $z_0$ lies in $K$, and $z_0\ne z_\star$ (else $\cE_0=0$), with
$z_0-z_\star\in H$. Since $z_0-z_\star\in H$, projecting $r_0$ onto the line
through it gives
\begin{equation*}
  \|P_H(r_0)\|_2
  \ge\frac{\E[(p_0-p_\star)(z_0-z_\star)]}{\|z_0-z_\star\|_2}
  \ge\mu\|z_0-z_\star\|_2,
\end{equation*}
using $\E[(p_0-p_\star)(z_0-z_\star)]\ge\mu\|z_0-z_\star\|_2^2$ from
\Cref{lem:classification-strong-convexity}. The same lemma gives
$\cE_0\le\tfrac18\|z_0-z_\star\|_2^2$, so $\|z_0-z_\star\|_2^2\ge8\cE_0$ and
$\|P_H(r_0)\|_2^2\ge8\mu^2\cE_0$. Substituting into
\Cref{eq:classification-r0-vs-drop} gives $8\mu^2\cE_0\le(M^2+1)\Delta/\lambda$,
that is $\Delta\ge8\mu^2\lambda\,\cE_0/(M^2+1)$. With $\Delta=\cE_0-\cE_M$,
\begin{equation*}
  \cE_M\le\Bigl(1-\frac{8\mu^2\lambda}{M^2+1}\Bigr)\cE_0=q\,\cE_0 .
\end{equation*}
Since $\mu\le\tfrac14$ and $\lambda\le M$ (because
$\sum_i\|P_{V_i}(u)\|_2^2\le M\|u\|_2^2$), we have $8\mu^2\lambda\le M/2<M^2+1$, so
$q\in[0,1)$, depending only on $\cD$ and $(S_1,\ldots,S_M)$.
\end{proof}

Only finitely many feature tuples can arise along a path, so taking the worst
one gives a single contraction factor, exactly as in
\Cref{thm:fixed-distribution-geometric}.

\begin{theorem}[Classification fixed-distribution geometric convergence]
\label{thm:classification-fixed-geometric}
Fix $M$ and a distribution $\cD$ on $(x_1,\ldots,x_d,Y)$ with $d$ finite,
$Y\in\{0,1\}$, and BCE minimum over $H$ attained at $z_\star$. Assume
$x_1,\ldots,x_d$ have bounded second moments. There is a
constant $q\in[0,1)$, depending only on $\cD$ and $M$, such that for any DAG of
agents on $\cD$, if $A_1\to\cdots\to A_D$ is a path whose every $M$ consecutive
raw-feature spans sum to $H$, then
\begin{equation*}
  \cE_{s+M}\le q\,\cE_s
  \qquad\text{whenever }1\le s\le D-M,
\end{equation*}
and consequently $\cE_D\le\cE_1\,q^{\lfloor(D-1)/M\rfloor}$.
\end{theorem}

\begin{proof}
Fix $s$ with $1\le s\le D-M$. The sub-path $A_{s+1}\to\cdots\to A_{s+M}$ is a
block of $M$ consecutive agents whose raw-feature spans sum to $H$, and its first
agent has the on-path predecessor $A_s$. Applying
\Cref{lem:classification-block-contraction} to this block gives
$\cE_{s+M}\le q(S_{s+1},\ldots,S_{s+M})\,\cE_s$ with a factor in $[0,1)$
depending only on $\cD$ and the tuple $(S_{s+1},\ldots,S_{s+M})$. That tuple takes
at most $2^{dM}$ values, so $q=\max q(S_{s+1},\ldots,S_{s+M})\in[0,1)$ over the
finitely many feasible tuples depends only on $\cD$ and $M$, and
$\cE_{s+M}\le q\,\cE_s$ for every such $s$.

The loss is non-increasing along the path (\Cref{lem:prelim-bce-kl}). Iterating
the per-block contraction from $s=1$ over $\lfloor(D-1)/M\rfloor$ full blocks and
using monotonicity on the remaining steps gives
$\cE_D\le\cE_1\,q^{\lfloor(D-1)/M\rfloor}$.
\end{proof}

\classificationNoFixedRestatement*

\begin{proof}
By \Cref{thm:classification-fixed-geometric},
$\cE_D\le\cE_1\,q^{\lfloor(D-1)/M\rfloor}$ for some $q\in[0,1)$ depending only on
$\cD$ and $M$. The zero logit is feasible for $A_1$, so
$\cE_1\le\cL(0)-\cL(z_\star)\le\log2$. Fix $c>0$. Since $q<1$, we have
$(\log2)\,Dq^{\lfloor(D-1)/M\rfloor}\to0$ as $D\to\infty$, so there is a $D_c$,
depending only on $\cD$, $M$, and $c$, with
$(\log2)\,Dq^{\lfloor(D-1)/M\rfloor}<cM^2$ for all $D\ge D_c$. Then
$\cL(z_D)-\cL(z_\star)=\cE_D\le(\log2)\,q^{\lfloor(D-1)/M\rfloor}<cM^2/D$.
\end{proof}

\printbibliography

\appendix

\section{Omitted Proofs}
\label{sec:omitted-proofs}

We will restate and prove lemmas whose proof was omitted from the main text.

\lemKearnsEndPassShape*

\begin{proof}
We show by induction that after pass $p$ the residual involves only
$z_{k-p},\ldots,z_k$, that is positions $0,\ldots,p$. In pass $1$ the prediction
stays zero until the path reaches $X_k$, since $X_1,\ldots,X_{k-1}$ depend only
on $z_1,\ldots,z_{k-1}$ and so are independent of $Y=z_k$; fitting $z_k$ from
$X_k=z_k-z_{k-1}$ then leaves residual $(z_k+z_{k-1})/2$, which involves only
$z_{k-1},z_k$. For the step, assume the residual after pass $p$ involves only
$z_{k-p},\ldots,z_k$. In pass $p+1$, every feature before $X_{k-p}$ depends only
on $z_1,\ldots,z_{k-p-1}$, hence is orthogonal to both $Y$ and the incoming
prediction. The incoming residual is orthogonal to the incoming prediction and
to such a feature, so \Cref{lem:prelim-ls-projection} shows that adjoining the
feature does not change the projection. From $X_{k-p}=z_{k-p}-z_{k-p-1}$ on, the
only new latent variable that can enter is $z_{k-p-1}$, so the residual after
pass $p+1$ involves only $z_{k-p-1},\ldots,z_k$. This gives the stated form
$R_p=\sum_{j=0}^p r_j^{(p)}z_{k-j}$.

Fix $p\le k-1$. In every pass, $X_1=z_1$ is the first feature. When it is seen
in a pass $q\le p$, the incoming prediction is the end of pass $q-1$, which involves only $z_{k-q+1},\ldots,z_k$; since $q\le k-1$, none of
these is $z_1$, and $z_1$ is also orthogonal to $Y=z_k$. So $X_1$ receives
coefficient zero in every pass and never enters the prediction, which therefore
lies in $\operatorname{span}\{X_2,\ldots,X_k\}$. Each $X_i$ with $i\ge2$ has
$z$-coefficients summing to zero, while $Y=z_k$ sums to one, so the residual
coefficients sum to $\sum_{j=0}^p r_j^{(p)}=1$.

Since $R_p=\sum_{j=0}^p r_j^{(p)}z_{k-j}$ and the $z$'s are orthonormal,
$E_p=\E[R_p^2]=\sum_{j=0}^p (r_j^{(p)})^2$.

The final prediction $\widehat y_p=Y-R_p$ has coefficient vector $e_0-r^{(p)}$.
By self-orthogonality of the least-squares predictor
(\Cref{lem:prelim-ls-orthogonality}), the residual is orthogonal to the
prediction, so $0=\langle e_0-r^{(p)},r^{(p)}\rangle=r_0^{(p)}-\sum_j(r_j^{(p)})^2$,
giving $E_p=\sum_j(r_j^{(p)})^2=r_0^{(p)}$. The last agent also sees
$X_k=z_k-z_{k-1}$, with coefficient vector $e_0-e_1$, so, by \Cref{lem:prelim-ls-projection}, the residual is
orthogonal to $e_0-e_1$ as well:
$0=\langle r^{(p)},e_0-e_1\rangle=r_0^{(p)}-r_1^{(p)}$.
\end{proof}

\lemKearnsOneAgentUpdate*

\begin{proof}
Let $f$ be the incoming prediction. Since $f$ is a least-squares predictor, the
incoming residual is orthogonal to $f$ by \Cref{lem:prelim-ls-projection}. If the new feature is orthogonal to both
$Y$ and $f$, then the incoming residual is also orthogonal to the new feature.
By \Cref{lem:prelim-ls-projection}, the least-squares projection onto the span
of $f$ and the new feature is still $f$, and $b^+=b$.

Now suppose the feature is $g_j=e_{j-1}-e_j$. The new prediction has vector
$\alpha(e_0-b)+\gamma g_j$ for some scalars $\alpha,\gamma$. Hence
\begin{equation*}
  b^+=e_0-\bigl(\alpha(e_0-b)+\gamma g_j\bigr)
  =(1-\alpha)e_0+\alpha b-\gamma g_j .
\end{equation*}
By \Cref{lem:prelim-ls-projection}, the new residual is orthogonal to $g_j$, so
$0=\langle b^+,g_j\rangle=b^+_{j-1}-b^+_j$.

For an interior feature, $j\ge2$, the $e_0$ term does not touch positions
$j-1$ and $j$. Reading off these two positions gives
\begin{equation*}
  b^+_{j-1}=\alpha b_{j-1}-\gamma,
  \qquad
  b^+_j=\alpha b_j+\gamma .
\end{equation*}
The condition $b^+_{j-1}=b^+_j$ becomes
$\alpha b_{j-1}-\gamma=\alpha b_j+\gamma$, hence
$\gamma=\alpha(b_{j-1}-b_j)/2$. Substituting this value back gives
\begin{equation*}
  b^+_{j-1}=b^+_j=\frac{\alpha}{2}(b_{j-1}+b_j).
\end{equation*}
The value at every other positive position is multiplied by the same factor
$\alpha$. Thus an interior update replaces the two touched positions by their
average, up to the common factor $\alpha$.

For the boundary feature $j=1$, the feature is $g_1=e_0-e_1$. Now
\begin{equation*}
  b^+_0=(1-\alpha)+\alpha b_0-\gamma,
  \qquad
  b^+_1=\alpha b_1+\gamma ,
\end{equation*}
and all positions $m\ge2$ become $\alpha b_m$. The condition $b^+_0=b^+_1$
reads $(1-\alpha)+\alpha b_0-\gamma=\alpha b_1+\gamma$, so
\begin{equation*}
  \gamma=\frac12\bigl[(1-\alpha)+\alpha(b_0-b_1)\bigr].
\end{equation*}
Substituting this value back gives
\begin{equation*}
  b^+_0=b^+_1=\frac12\bigl[(1-\alpha)+\alpha(b_0+b_1)\bigr].
\end{equation*}
The boundary update therefore equalizes positions $0$ and $1$, while the values
at positions $m\ge2$ are all multiplied by the same factor $\alpha$.
\end{proof}

\lemKearnsTailShape*

\begin{proof}
For $p=2$, pass $1$ leaves residual $(z_k+z_{k-1})/2$, so
$r^{(1)}=(\tfrac12,\tfrac12)$. In pass $2$, the first feature that changes the
prediction is $X_{k-1}=z_{k-1}-z_{k-2}$. The incoming prediction lies along
$X_k=z_k-z_{k-1}$, so by \Cref{lem:prelim-ls-projection} the new predictor is the
projection of $z_k$ onto
$\operatorname{span}\{z_k-z_{k-1},z_{k-1}-z_{k-2}\}$. Write the residual at this
point as $q_0z_k+q_1z_{k-1}+q_2z_{k-2}$. Orthogonality to $X_k$ and $X_{k-1}$
gives
\begin{equation*}
  0=\langle q,X_k\rangle=q_0-q_1,
  \qquad
  0=\langle q,X_{k-1}\rangle=q_1-q_2.
\end{equation*}
Thus $q_0=q_1=q_2$. The current prediction lies in the span of
$X_k$ and $X_{k-1}$, and both directions have coefficient sum zero. Since
$Y=z_k$ has coefficient sum one, the residual coefficients also sum to one, so
$R_2=(z_k+z_{k-1}+z_{k-2})/3$. This residual is already orthogonal to the
current prediction and to $X_k=z_k-z_{k-1}$, so
\Cref{lem:prelim-ls-projection} shows that the last agent of the pass leaves it
unchanged. Thus
$r^{(2)}=(\tfrac13,\tfrac13,\tfrac13)$. Since $\mu^{(2)}=(1)$ and $S_2=1$, the
formula gives $r^{(2)}$ and $E_2=r_0^{(2)}=S_2/(1+2S_2)$.

Now fix $3\le p\le k-1$ and assume the formula holds after pass $p-1$. Then
\begin{equation*}
  r^{(p-1)}
  =
  (1+2S_{p-1})^{-1}
  (S_{p-1},S_{p-1},\mu^{(p-1)}).
\end{equation*}
In pass $p$, all features before $g_p=e_{p-1}-e_p$ leave the prediction
unchanged by \Cref{lem:kearns-one-agent-update}: their nonzero positions are
outside $0,\ldots,p-1$ and do not include the target position $0$.

It remains to follow $g_j=e_{j-1}-e_j$ for $j=p,p-1,\ldots,1$. By
\Cref{lem:kearns-one-agent-update}, the updates with $j\ge2$ average the two
touched positive positions, up to a common rescaling, and the final update
$j=1$ only rescales positions $2,3,\ldots,p$. We may ignore these common
factors while computing the tail direction, since they will be absorbed into one
scalar at the end.

Before these updates, the positive positions $1,\ldots,p$ are proportional to
$v=(S_{p-1},\mu_1^{(p-1)},\ldots,\mu_{p-2}^{(p-1)},0)$, where $v_i$ is the
unscaled value at position $i$. Let $u_i$ be the final unscaled value at
position $i+1$, and set $u_p:=0$. When $g_{i+1}$ is reached, position $i$ still
has value $v_i$. The current value at position $i+1$ is $u_{i+1}$: for
$i=p-1$ this is the new zero position, and for $i<p-1$ it was made equal to
position $i+2$ by the previous update and position $i+2$ is never touched again.
Thus averaging positions $i$ and $i+1$ gives
\begin{equation*}
  u_i=\frac12 v_i+\frac12 u_{i+1}
  \qquad (1\le i\le p-1).
\end{equation*}
Unrolling the recursion gives
\begin{equation*}
  u_i=\sum_{m=i}^{p-1}2^{-(m-i+1)}v_m .
\end{equation*}
Thus the mass $S_{p-1}$ at position $1$ adds $S_{p-1}/2$ to $u_1$. For every
$1\le m\le p-2$, the mass $\mu_m^{(p-1)}$ sits at position $m+1$ and adds
$2^{-(m+2-i)}\mu_m^{(p-1)}$ to $u_i$ for $1\le i\le m+1$. This is exactly the
defining rule for $u^{(p)}$. It has positive total mass because
$\mu^{(p-1)}$ is a probability vector, so $\mu^{(p)}$ is defined.

Restoring the common scale gives
$r^{(p)}=(a,a,c'\mu_1^{(p)},\ldots,c'\mu_{p-1}^{(p)})$ for some scalars
$a,c'$.

It remains to determine $a$ and $c'$. Put
$S=S_p=\sum_i(\mu_i^{(p)})^2$. The coefficient-sum identity from
\Cref{lem:kearns-end-pass-shape} gives $2a+c'=1$. The same lemma gives
$E_p=r_0^{(p)}=a$, while orthonormality gives $E_p=2a^2+(c')^2S$. Thus
\begin{equation*}
  a=2a^2+(c')^2S.
\end{equation*}
Substituting $c'=1-2a$ gives
\begin{equation*}
  a=2a^2+(1-2a)^2S,
\end{equation*}
or equivalently
\begin{equation*}
  0=(2+4S)a^2-(1+4S)a+S.
\end{equation*}
The two roots are
$a=\tfrac12$ and $a=S/(1+2S)$. The error is non-increasing along the path by
\Cref{lem:prelim-mse-decomposition}; since $E_2=\tfrac13$ and $p\ge3$, we have
$a=E_p\le\tfrac13$, so $a=\tfrac12$ is impossible. Hence
$a=S/(1+2S)$ and $c'=1-2a=1/(1+2S)$. This is exactly the claimed formula at pass
$p$, and the induction is complete.
\end{proof}

\lemKearnsTailMixture*

\begin{proof}
We prove this by induction on $p$. For $p=2$, the claim is
$\mu^{(2)}=\nu_0$.

Assume the claim holds for some $2\le p\le k-2$. Write
$\mu^{(p)}=\sum_{t=0}^{p-2}\lambda_t\nu_t$, where $\lambda_t\ge0$ and
$\sum_{t=0}^{p-2}\lambda_t=1$. Let $K$ be one unnormalized killed-walk step:
if $\eta$ is a vector on the positive integers, then $(K\eta)_i$ is the mass at
integer $i$ after one step, with all mass that lands at $0$ or below removed.
Equivalently, $(K\eta)_i=\sum_{j\ge1,\ i\le j+1}\eta_j2^{-(j+2-i)}$.
Hence, by \Cref{lem:kearns-tail-shape}, the
unnormalized vector that is normalized to form $\mu^{(p+1)}$ is
$K\mu^{(p)}+(S_p/2)\nu_0$: the second term is the extra mass placed at integer
$1$, since $\nu_0$ is the unit mass at integer $1$.

It remains to see what $K$ does to each $\nu_t$. By definition of $\nu_t$,
starting from $\nu_t$ is the same as conditioning on $\tau>t$. Thus, for
$i\ge1$,
\begin{align*}
  (K\nu_t)_i
  &=
  \Pr(W_{t+1}=i,\tau>t+1\mid \tau>t)\\
  &=
  \Pr(\tau>t+1\mid \tau>t)
  \Pr(W_{t+1}=i\mid \tau>t+1,\tau>t)\\
  &=
  \Pr(\tau>t+1\mid \tau>t)
  \Pr(W_{t+1}=i\mid \tau>t+1)\\
  &=
  \Pr(\tau>t+1\mid \tau>t)\nu_{t+1}(i).
\end{align*}
The equality that removes the condition $\tau>t$ uses that $\tau>t+1$ implies
$\tau>t$. Therefore $K\nu_t=q_t\nu_{t+1}$, where
$q_t=\Pr(\tau>t+1\mid\tau>t)\ge0$.

Substituting the induction hypothesis gives
$$K\mu^{(p)}+(S_p/2)\nu_0=(S_p/2)\nu_0+\sum_{t=0}^{p-2}\lambda_tq_t\nu_{t+1}.$$
This is a nonnegative combination of $\nu_0,\ldots,\nu_{p-1}$. Normalizing this combination gives a convex combination of the same vectors. Thus
$\mu^{(p+1)}$ is a convex combination of $\nu_0,\ldots,\nu_{p-1}$.
\end{proof}

\lemKearnsCatalanPrefixGF*

\begin{proof}
Consider a valid word $w$. Since the word is valid, we can decompose it into a sequence of valid balanced words separated by opening parentheses. Suppose that $o(w) - c(w) = h$. These $h$
extra opening parentheses split the word uniquely as
\begin{equation*}
  w=u_0\mathtt{(} u_1\mathtt{(} \cdots \mathtt{(}u_h.
\end{equation*}
where each $u_i$ is a valid balanced word. Summing over all possible values of $h$ gives
\begin{align*}
  G(a,b) &= \sum_{h\ge0} F(a,b)(aF(a,b))^h = \sum_{h\ge0} C(ab)(aC(ab))^h \\
  &= C(ab) \sum_{h\ge0} (aC(ab))^h = \frac{C(ab)}{1-aC(ab)}. \qedhere
\end{align*}

\end{proof}

\lemKearnsSurvivalWords*

\begin{proof}
For a prefix $u$ of $w$, consider the value $1+o(u)-c(u)$. At the end of the block for
step $s$, this value is $W_s$.

If the mass is killed by time $t$, then $W_s\le0$ for some $s\le t$, so the
prefix ending at that block has more closing than opening parentheses. Thus $w$
is not valid. Conversely, suppose $w$ is not valid, and take the first prefix
$u$ with $o(u)<c(u)$. This first failure occurs during the closing parentheses
of some step $s$. The rest of that block also consists of closing parentheses,
so $W_s\le 1+o(u) - c(u)\le0$. Hence the mass is killed by time $t$.
\end{proof}

\lemKearnsWalkSurvival*

\begin{proof}
By the definition of $H(z)$ and \Cref{lem:kearns-central-binomial-coefficient}, we have $\Pr(\tau>t) = \binom{2t+2}{t+1}/2^{2t+1}$.
Set $n=t+1$. We use the following standard bounds:
\begin{equation*}
  4^n/(2\sqrt n)\le\binom{2n}{n}\le4^n/\sqrt n.
\end{equation*}
Since $2^{2t+1}=2^{2n-1}=4^n/2$, substituting $n=t+1$ in the exact formula gives
the two displayed bounds on $\Pr(\tau>t)$.
\end{proof}

\lemKearnsIncrementMoments*

\begin{proof}
From $\sum_{\ell\ge0}q^\ell=1/(1-q)$, differentiating once and multiplying by
$q$ gives $\sum_{\ell\ge0}\ell q^\ell=q/(1-q)^2$. Differentiating this identity
and multiplying by $q$ gives
$\sum_{\ell\ge0}\ell^2q^\ell=q(1+q)/(1-q)^3$. With $q=1/2$,
$\E[G_t]=\sum_{\ell\ge0}\ell 2^{-(\ell+1)}=1$ and
$\E[G_t^2]=\sum_{\ell\ge0}\ell^2 2^{-(\ell+1)}=3$.
Thus $\E[1-G_t]=0$ and
$\operatorname{Var}(1-G_t)=\operatorname{Var}(G_t)=3-1^2=2$.
\end{proof}

\lemKearnsStoppedTimeMean*

\begin{proof}
For each outcome, $\min\{t,\tau\}$ is the number of integers
$s\in\{0,\ldots,t-1\}$ for which $\tau>s$. Taking expectations gives
\begin{equation*}
\E[\min\{t,\tau\}]=\sum_{s=0}^{t-1}\Pr(\tau>s)
\le\sum_{s=0}^{t-1}2/\sqrt{s+1}\le4\sqrt t.
\end{equation*}
The last inequality is trivial for $t=0$ and follows by comparison with an
integral for $t\ge1$:
\begin{equation*}
\sum_{s=0}^{t-1}1/\sqrt{s+1}\le1+\int_1^t x^{-1/2}\,dx\le2\sqrt t. \qedhere
\end{equation*}
\end{proof}

\section{The Upper Bound Assumptions}
\label{sec:scale-assumptions}

\Cref{thm:regression-improved-upper} makes two assumptions: the global
predictor $f^\star(x)=\sum_{\ell=1}^d w^\star_\ell x_\ell$ satisfies
$\sum_\ell|w^\star_\ell|\le A^\star$, and each feature satisfies
$\E[x_\ell^2]\le M_X^2$. In this section we show that neither assumption can
be dropped. If either one is removed, the excess error of the last agent can be
made arbitrarily large, even when $M$ and $D$ are fixed.

Both results start from the instance of \Cref{thm:long-depth} and multiply the
label by a constant $a$. This multiplies every path predictor by $a$, since
each agent projects the label onto a span that does not change. Thus the excess
error is multiplied by $a^2$, and it grows without bound as $a$ grows.

\begin{proposition}[The coefficient bound cannot be dropped]
\label{prop:coefficient-bound-needed}
For every $M\ge8$, every $D\ge M^2$, and every $R>0$, there is an $M$-covered
path of depth $D$ whose global predictor $f^\star$ is exact and whose features
satisfy $\E[x_\ell^2]\le2$, but
\begin{equation*}
  \MSE(f_D)-\MSE(f^\star)
  \ge
  R.
\end{equation*}
\end{proposition}

\begin{proof}
Take the instance of \Cref{thm:long-depth} for these $M$ and $D$. Its
features satisfy $\E[x_\ell^2]\le2$, its global predictor $f^\star$ is exact,
and its last agent has excess error
$E_D:=\MSE(f_D)-\MSE(f^\star)\ge M^2/(1280\pi^2D)>0$. Set $a=\sqrt{R/E_D}$,
keep the features, and replace the label by $Y'=aY$. The global predictor is still
exact, since $Y'=af^\star$ is a linear combination of the features, and it equals $af^\star$.

Let $f'_1,\ldots,f'_D$ be the path predictors for the label $Y'$. We induct on $t$ to show that $f'_t = af_t$. Agent $A_t$ sees the features $x_{S_t}$ as before and, by the induction hypothesis, the prediction $af_{t-1}$ in place of $f_{t-1}$. These inputs span the same set as before, so by \Cref{lem:prelim-ls-projection}, $f'_t$ is the projection of $aY$ onto the same span as $f_t$, which is $af_t$.

Since $f'_D=af_D$ and the global predictor is $af^\star$, we have
\begin{equation*}
  \MSE(f'_D)-\MSE(af^\star)
  =
  a^2\bigl(\MSE(f_D)-\MSE(f^\star)\bigr)
  =
  a^2E_D
  =
  R.
  \qedhere
\end{equation*}
\end{proof}

\begin{proposition}[The second-moment bound cannot be dropped]
\label{prop:moment-bound-needed}
For every $M\ge8$, every $D\ge M^2$, and every $R>0$, there is an $M$-covered
path of depth $D$ whose global predictor
$f^\star(x)=\sum_{\ell=1}^d w^\star_\ell x_\ell$ is exact and satisfies
$\sum_\ell|w^\star_\ell|\le3$, but
\begin{equation*}
  \MSE(f_D)-\MSE(f^\star)
  \ge
  R.
\end{equation*}
\end{proposition}

\begin{proof}
Take the same instance of \Cref{thm:long-depth} and the same
$a=\sqrt{R/E_D}$ as in the proof of \Cref{prop:coefficient-bound-needed}.
Its global predictor $f^\star=\sum_\ell w^\star_\ell x_\ell$ is exact with
$\sum_\ell|w^\star_\ell|\le3$. Now multiply every feature and the label by
$a$, so the new features are $x'_\ell=ax_\ell$ and the new label is
$Y'=aY$. The coefficients $w^\star$ still give an exact global predictor,
since
\begin{math}
  Y'
  =
  a\sum_\ell w^\star_\ell x_\ell
  =
  \sum_\ell w^\star_\ell x'_\ell,
\end{math}
so the new global predictor is $af^\star$.

We induct on $t$ to show that $f'_t=af_t$. Agent $A_t$ now sees the
features $x'_{S_t}=ax_{S_t}$ and, by the induction hypothesis, the prediction
$af_{t-1}$. These are the old inputs multiplied by $a$, so they span the same
set as before, and as in the proof of \Cref{prop:coefficient-bound-needed},
$f'_t$ is the projection of $aY$ onto that span, which is $af_t$. Hence
$\MSE(f'_D)-\MSE(af^\star)=a^2E_D=R$.
\end{proof}

\end{document}